\documentclass{applemlr}

\usepackage{amsmath}
\usepackage{enumerate}
\usepackage{algorithm}
\usepackage{algpseudocode}
\usepackage{amsfonts}
\usepackage{amsthm}
\usepackage{cleveref}
\usepackage{diagbox}
\usepackage{colortbl}
\usepackage{amssymb}
\usepackage{xspace}
\usepackage{wrapfig}
\usepackage{adjustbox}
\usepackage{tabularx}
\usepackage{booktabs}
\usepackage{mathtools}
\usepackage{tikz}
\usepackage{enumitem}
\usepackage{silence}
\usepackage{dsfont}
\usepackage[table]{xcolor}
\usepackage[dvipsnames]{xcolor}
\usepackage{multirow}
\usepackage{makecell}
\usepackage{xfakebold}

\definecolor{textgray}{HTML}{6E6E73}
\usetikzlibrary{positioning, calc}
\usetikzlibrary{decorations.pathmorphing}

\makeatletter
\patchcmd{\wrong@fontshape}{\@gobbletwo}{}{}{}
\makeatother
\numberwithin{equation}{section}
\makeatletter
\AtBeginDocument{
  \urlstyle{sf}
  
}
\makeatother

\definecolor{light}{RGB}{125, 125, 125}
\crefname{tcb@cnt@pbox}{code}{code}
\Crefname{tcb@cnt@pbox}{Code}{Code}
\crefname{assumption}{assumption}{assumption}
\Crefname{assumption}{Assumption}{Assumptions}

\newtcolorbox[auto counter]{pbox}[2][]{
  colback=white,
  title=Code~\thetcbcounter: #2,
  #1,fonttitle=\sffamily,
  fontupper=\sffamily,
  arc=2pt,
  colframe=bgcolor,
  coltitle=fgcolor,
  colbacktitle=bgcolor,
  toptitle=0.25cm,
  bottomtitle=0.125cm
}

\makeatletter
\newcommand\applefootnote[1]{%
  \begingroup
  \renewcommand\thefootnote{}%
  \renewcommand\@makefntext[1]{\noindent##1}%
  \footnote{#1}%
  \addtocounter{footnote}{-1}%
  \endgroup
}
\makeatother

\definecolor{cverbbg}{gray}{0.90}

\newcommand*\methodname{NTM}
\newcommand*\fullname{Normalizing Trajectory Models}

\NewDocumentCommand{\todo}{ mO{} }{\textcolor{red}{\textsuperscript{\textit{TODO}}\textsf{\textbf{\small[#1]}}}}

\usepackage{bm}
\ifcsname proposition\endcsname\else\newtheorem{proposition}{Proposition}\fi

\providecommand{\REQUIRE}{\Require}

\providecommand{\STATE}{\State}
\providecommand{\RETURN}{\State \textbf{return} }
\providecommand{\COMMENT}[1]{\Comment{#1}}
\providecommand{\FOR}{\For}
\providecommand{\ENDFOR}{\EndFor}

\usepackage{inconsolata}
\crefformat{equation}{Eq.~(#2#1#3)}
\crefname{section}{\S}{\S\S}
\Crefname{section}{\S}{\S\S}
\usepackage{nicefrac}
\usepackage{graphicx}
\usetikzlibrary{arrows.meta,decorations.pathreplacing,shapes.geometric}
\usepackage{caption}
\usepackage{subcaption}
\usepackage{pifont}
\usepackage{listings}
\usepackage{array}
\usepackage{environ}
\usepackage{float}

\usepackage{amsmath,amsfonts,bm}

\def\eqref#1{equation~\ref{#1}}

\def\1{\bm{1}}

\def\vg{{\bm{g}}}
\def\vh{{\bm{h}}}

\def\vt{{\bm{t}}}
\def\vu{{\bm{u}}}
\def\vv{{\bm{v}}}

\def\vx{{\bm{x}}}
\def\vy{{\bm{y}}}
\def\vz{{\bm{z}}}

\def\mI{{\bm{I}}}

\DeclareMathAlphabet{\mathsfit}{\encodingdefault}{\sfdefault}{m}{sl}
\SetMathAlphabet{\mathsfit}{bold}{\encodingdefault}{\sfdefault}{bx}{n}

\title{\fullname{}}

\author[1]{Jiatao Gu}
\author[1]{Tianrong Chen}
\author[2]{Ying Shen}
\author[1]{David Berthelot}
\author[1]{Shuangfei Zhai}
\author[1]{Josh Susskind}

\affiliation[1]{Apple}
\affiliation[2]{UIUC}

\abstract{Diffusion-based models decompose sampling into many small Gaussian denoising steps, an assumption that breaks down when generation is compressed to a few coarse transitions.
Existing few-step methods address this through distillation, consistency training, or adversarial objectives, but sacrifice the likelihood framework in the process.
We introduce \fullname{} (\methodname{}), which models each reverse step as an expressive conditional normalizing flow with exact likelihood training.
Architecturally, \methodname{} combines shallow invertible blocks within each step with a deep parallel predictor across the trajectory, forming an end-to-end network trainable from scratch or initializable from pretrained flow-matching models.
Its exact trajectory likelihood further enables self-distillation: a lightweight denoiser trained on the score function induced by the model itself produces high-quality samples in four steps.
On text-to-image benchmarks, \methodname{} matches or outperforms strong image generation baselines in just four sampling steps while uniquely retaining exact likelihood over the generative trajectory.}

\metadata[Code]{\url{https://github.com/apple/ml-starflow}}
\metadata[Correspondence]{\sffamily  \email{jgu32@apple.com}}
\date{\sffamily\today}

\begin{document}

\maketitle

\applefootnote{Work done while JG holding a joint affiliation at University of Pennsylvania, and YS working as a research intern at Apple.}

\begin{figure}[h!]
\centering
\includegraphics[width=\textwidth]{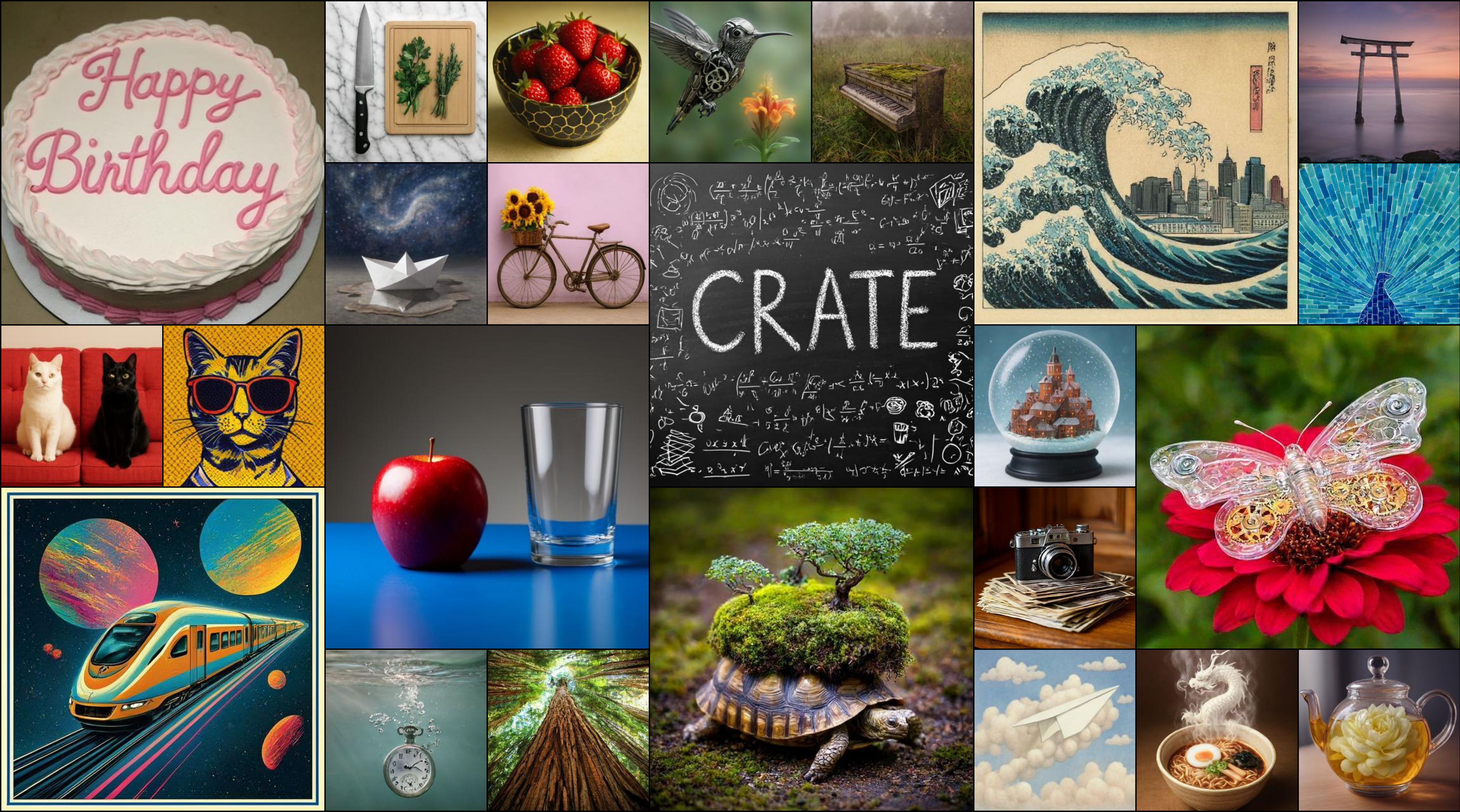}
\caption{
\textbf{Text-to-image generation with \methodname{} with $4$ denoising steps.}
We show samples from models trained from scratch at 256$\times$256, and from models obtained by finetuning pretrained flow-matching checkpoints at 512$\times$512.
}
\label{fig:teaser}
\end{figure}

\section{Introduction}

Diffusion-based models~\citep{ho2020denoising,song2021scorebased,lipman2023flow,liu2023flow,albergo2023stochastic} have become the dominant paradigm for high-fidelity image generation~\citep{rombach2022high,esser2024scaling,podell2023sdxl}.
These methods decompose generation into many small denoising steps, each modeled as a Gaussian transition whose mean is predicted by a neural network.
When the step size is small, this Gaussian approximation is accurate: the reverse conditional $p(\vx_s \mid \vx_t)$ is close to Gaussian because the transition covers only a small portion of the diffusion trajectory.
However, reducing the number of sampling steps to improve efficiency forces each transition to span a larger interval, and the true reverse conditional becomes a mixture of Gaussians that can be multimodal and heavy-tailed.
The single-Gaussian assumption then becomes a fundamental bottleneck for few-step generation quality.

A growing body of work addresses the efficiency problem, but existing approaches sacrifice the likelihood framework.
Distillation methods~\citep{salimans2022progressive,yin2024one} and consistency models~\citep{song2023consistency,luo2023latent} learn to map noise to data in fewer steps, yet provide no tractable density over the generative trajectory.
DDGAN~\citep{xiao2022ddgan} replaces the Gaussian reverse with an implicit distribution learned via adversarial training, but introduces mode-seeking behavior and training instability that limit scalability.
No existing method achieves few-step generation with an exact likelihood model of the reverse process.

We introduce \textbf{\fullname{} (\methodname{})}, a framework that models $p(\vx_s \mid \vx_t)$ as a conditional normalizing flow with exact log-likelihood.
The core idea is to learn a latent space---via an invertible transporter---where the reverse conditional becomes simple enough to be modeled by a Gaussian predictor.
Unlike a compressive encoder, the transporter preserves dimensionality and invertibility, which together with the Gaussian predictor yields exact log-likelihood training through the change-of-variables formula.
This bridges self-supervised representation learning and probabilistic generative modeling: the framework resembles a predictor--encoder architecture~\citep{grill2020bootstrap,assran2023self}, but the invertibility constraint turns it into a normalizing flow.

\begin{figure}[h]
\centering
\includegraphics[width=\linewidth]{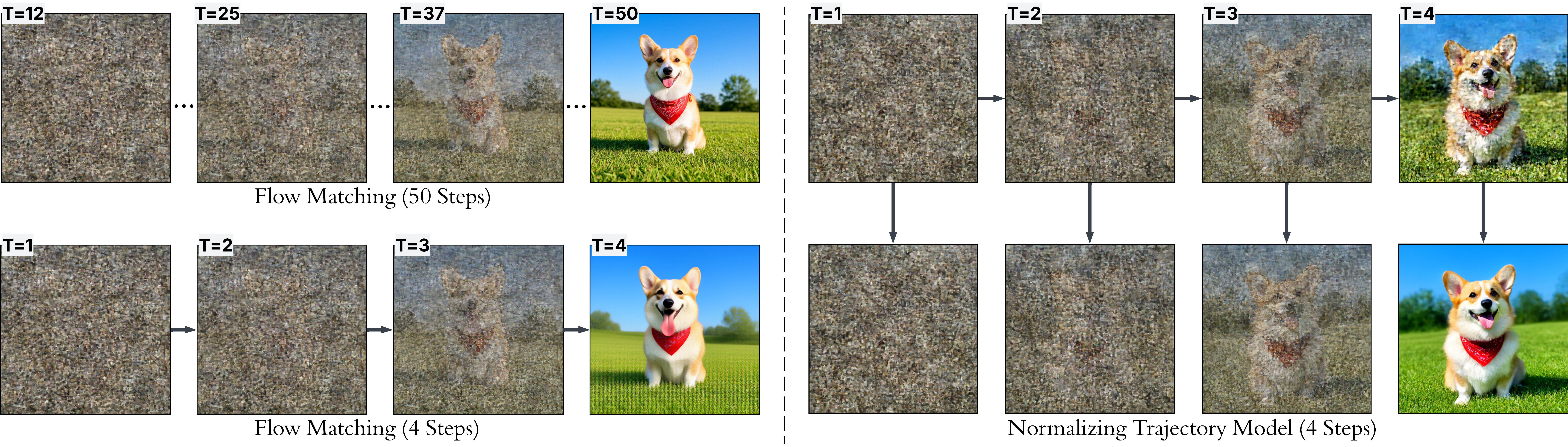}
\caption{\textbf{Denoising trajectories.}
\textbf{Left}: Flow matching with 50 steps and 4 steps.
\textbf{Right}: \methodname{} achieves comparable quality in 4 steps by modeling the non-Gaussian reverse conditional.}
\label{fig:trajectory_comparison}
\end{figure}

\methodname{} can be trained from scratch using stochastic forward trajectories, or initialized from any pretrained flow-matching model by setting the transporter to identity and the predictor to the pretrained Gaussian posterior.
The exact trajectory likelihood further enables score-based denoising: since the generated trajectory is an inherently noisy sequence from the Markov forward process, the gradient of the NTM loss provides a joint score that denoises all timesteps simultaneously by exploiting their correlations.
A lightweight learned denoiser can distill this signal into a single forward pass, producing high-quality samples in as few as four steps.
Experiments on class-conditional and text-to-image generation demonstrate that \methodname{} matches or outperforms strong few-step baselines in image quality and compositional accuracy, achieving 0.82 on GenEval~\citep{ghosh2023geneval} with only 4 denoising steps when trained from scratch---significantly outperforming the prior normalizing flow model STARFlow (0.56, requiring 256 AR steps)---while uniquely retaining exact likelihood over the generative trajectory.

Our contributions are:
\begin{itemize}[leftmargin=*,nosep]
    \item A framework that models the non-Gaussian reverse conditional $p(\vx_s \mid \vx_t)$ via an invertible transporter and a Gaussian predictor, yielding exact log-likelihood while bridging representation learning and probabilistic modeling.
    \item A finetuning recipe that initializes from pretrained diffusion or flow-matching models via identity transporter and zero-initialized scale correction, preserving pretrained quality at initialization.
    \item Score-based trajectory denoising that exploits the exact likelihood and Markov covariance to jointly correct generated trajectories, distillable into a learned denoiser for four-step generation without additional training data.
\end{itemize}

\section{Preliminaries}
\label{sec:prelim}

\subsection{Flow Matching and Diffusion Models}
\label{sec:prelim_fm}
Flow matching~\citep{lipman2023flow,liu2023flow,albergo2023stochastic} defines a forward interpolation between clean data $\vx_0$ and Gaussian noise $\bm{\epsilon} \sim \mathcal{N}(\mathbf{0}, \mI)$:
\begin{equation}
\label{eq:forward}
    \vx_t = (1-t)\,\vx_0 + t\,\bm{\epsilon}, \quad q(\vx_t \mid \vx_0) = \mathcal{N}\!\big((1-t)\,\vx_0,\; t^2 \mI\big), \quad t \in [0, 1].
\end{equation}
A neural network $v_\theta(\vx_t, t)$ is trained to predict the velocity field by minimizing
\begin{equation}
\label{eq:fm_loss}
    \mathcal{L}_{\mathrm{FM}} = \mathbb{E}_{t,\, \vx_0,\, \bm{\epsilon}} \big\| v_\theta(\vx_t, t) - (\bm{\epsilon} - \vx_0) \big\|^2,
\end{equation}
and samples are generated by integrating the learned ODE $\mathrm{d}\vx = v_\theta(\vx, t)\,\mathrm{d}t$ from $t{=}1$ (noise) to $t{=}0$ (data).
Mathematically, diffusion models~\citep{ho2020denoising} can be designed to share the same marginals $q(\vx_t \mid \vx_0)$ under equivalent noise schedules, but define a stochastic forward process whose discretized reverse takes the form of a Gaussian transition kernel $p_\theta(\vx_s \mid \vx_t) = \mathcal{N}(\bm{\mu}_\theta(\vx_t, t, s),\, \sigma^2(t,s)\,\mI)$.

In both frameworks, generation quality depends on the number of discretization steps: flow matching assumes the velocity field is locally linear within each step, while diffusion models assume the reverse conditional is Gaussian.
With many steps these approximations are accurate; with few steps each transition must cover a large interval, and the true mapping from $\vx_t$ to $\vx_s$ becomes too complex for either a linear or Gaussian model to capture.
To formalize and address this limitation, we adopt a stochastic trajectory framework that makes the per-step distribution an explicit modeling target.

\subsection{Stochastic Trajectories and the Gaussian Bottleneck}
\label{sec:prelim_traj}
Given a timestep schedule $0 = t_0 < t_1 < \cdots < t_T = 1$, we construct a Markovian forward trajectory that satisfies the marginal constraint in \cref{eq:forward} at every step.
For any two consecutive timesteps $s < t$ in the schedule, the forward transition is:
\begin{equation}
\label{eq:markov_transition}
    \vx_t = \alpha_{s,t} \, \vx_s + \sigma_{s,t} \, \bm{\epsilon}, \quad \alpha_{s,t} = \frac{1 - t}{1 - s}, \quad \sigma_{s,t} = \sqrt{t^2 - \alpha_{s,t}^2 \, s^2},
\end{equation}
where $\bm{\epsilon} \sim \mathcal{N}(\mathbf{0}, \mI)$.
Applying this transition sequentially yields a correlated stochastic path $(\vx_{t_0}, \vx_{t_1}, \ldots, \vx_{t_T})$ from near-clean to near-noise, with each point marginally distributed as $q(\vx_t \mid \vx_0)$.
The Markovian structure defines a tractable joint distribution over the trajectory whose reverse conditionals $q(\vx_s \mid \vx_t, \vx_0)$ are Gaussian with known mean and variance.

\paragraph{The Gaussian approximation.}
\label{sec:prelim_limitation}
Standard diffusion and flow-matching models approximate the reverse conditional $p(\vx_s \mid \vx_t)$ with a single Gaussian $\mathcal{N}(\bm{\mu}_\theta(\vx_t), \sigma^2 \mI)$.
This is exact for the posterior conditioned on the clean image, $p(\vx_s \mid \vx_t, \vx_0)$, which is Gaussian by construction of the Markovian forward process.
However, the marginal reverse conditional integrates over all possible clean images:
\begin{equation}
\label{eq:marginal_reverse}
    p(\vx_s \mid \vx_t) = \int p(\vx_s \mid \vx_t, \vx_0)\, p(\vx_0 \mid \vx_t)\, \mathrm{d}\vx_0.
\end{equation}
Since $p(\vx_0 \mid \vx_t)$ is complex and potentially multimodal over natural images, the marginal $p(\vx_s \mid \vx_t)$ is a mixture of Gaussians that a single Gaussian cannot capture.
When the number of steps is small, each transition spans a large interval and the approximation error becomes severe.

\subsection{Normalizing Flows}
\label{sec:prelim_nf}

Normalizing flows~\citep{dinh2014nice,rezende2015variational,dinh2016density,kingma2018glow} learn an invertible mapping $f_\theta: \mathbb{R}^D \to \mathbb{R}^D$ between data $\vx$ and a latent $\vz = f_\theta(\vx)$ drawn from a simple prior $p_0(\vz) = \mathcal{N}(\mathbf{0}, \mI)$.
The exact log-likelihood is given by the change-of-variables formula:
\begin{equation}
\label{eq:nf_nll}
    \log p(\vx) = \log p_0\!\big(f_\theta(\vx)\big) + \log \big|\det J_{f_\theta}(\vx)\big|.
\end{equation}
A common design is the autoregressive flow~\citep{kingma2016improved,papamakarios2017masked}, which transforms each element conditioned on all preceding elements via affine (NVP) coupling~\citep{dinh2016density}, yielding a tractable triangular Jacobian.
TarFlow~\citep{zhainormalizing} parameterizes the affine coupling with a causal Transformer: each spatial token $\vx_n$ is transformed conditioned on all preceding tokens $\vx_{<n}$ via a self-exclusive causal mask:
\begin{equation}
\label{eq:nvp}
    \vz_n = \frac{\vx_n - \bm{\mu}_\theta(\vx_{<n})}{\bm{\sigma}_\theta(\vx_{<n})}, \quad \log\big|\det J\big| = -\sum_n \log \bm{\sigma}_\theta^{(n)},
\end{equation}
where $\bm{\sigma}_\theta > 0$ (scale) and $\bm{\mu}_\theta$ (shift) are predicted from preceding tokens.
This allows normalizing flows to scale competitively for high-resolution image generation.
STARFlow~\citep{gu2025starflow} further introduces a \emph{deep-shallow} architecture: a single deep autoregressive flow block with many Transformer layers captures most of the model capacity, followed by a few lightweight shallow blocks with alternating scan directions (e.g., left-to-right and right-to-left) that refine spatial details.
This deep-shallow design, extended to video in STARFlow-V~\citep{gu2025starflowv}, forms the architectural foundation of \methodname{}.

\section{\fullname{}}
\label{sec:method}

We present \fullname{} (\methodname{}), a generative framework that models the full conditional distribution $p(\vx_s \mid \vx_t)$ at each denoising step as a normalizing flow with exact log-likelihood (\cref{sec:formulation}).
\methodname{} can be trained from scratch (\cref{sec:from_scratch}), finetuned from pretrained diffusion or flow-matching models (\cref{sec:from_pretrained}), and accelerated to real-time generation via a learned denoiser (\cref{sec:denoiser}).
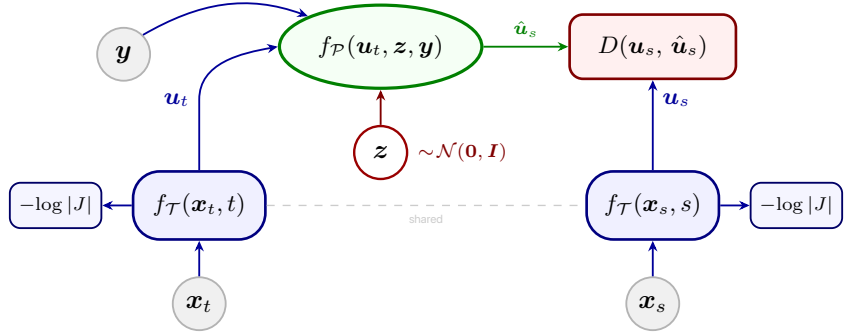
\begin{figure}[t]
    \centering
    \begin{tikzpicture}[
        scale=1, every node/.style={scale=1},
        enc/.style={rectangle, rounded corners=10pt, minimum width=1.75cm, minimum height=0.9cm,
            draw=blue!50!black, line width=1pt, fill=blue!6, font=\small\sffamily},
        pred/.style={ellipse, minimum width=2.6cm, minimum height=1.1cm,
            draw=green!50!black, line width=1.2pt, fill=green!4, font=\small\sffamily},
        dist/.style={rectangle, rounded corners=5pt, minimum width=2.2cm, minimum height=0.85cm,
            draw=red!50!black, line width=1pt, fill=red!6, font=\small\sffamily},
        jac/.style={rectangle, rounded corners=4pt, minimum width=1.2cm, minimum height=0.6cm,
            draw=blue!40!black, line width=0.7pt, fill=blue!4, font=\scriptsize},
        inpt/.style={circle, minimum size=0.7cm, draw=gray!60, line width=0.7pt, fill=gray!12, font=\normalsize, inner sep=1pt},
        znode/.style={circle, minimum size=0.7cm, draw=red!60!black, line width=0.9pt, fill=white, font=\normalsize, inner sep=1pt},
        arr/.style={-{Stealth[length=4pt,width=3.5pt]}, line width=0.7pt, blue!60!black},
        garr/.style={-{Stealth[length=4pt,width=3.5pt]}, line width=0.7pt, green!50!black},
        rarr/.style={-{Stealth[length=4pt,width=3.5pt]}, line width=0.7pt, red!50!black},
    ]
    \node[inpt] (xt) at (0,0) {$\vx_t$};
    \node[inpt] (xs) at (6,0) {$\vx_s$};

    \node[enc] (el) at (0,1.3) {$f_\mathcal{T}(\vx_t, t)$};
    \node[enc] (er) at (6,1.3) {$f_\mathcal{T}(\vx_s, s)$};

    \node[jac] (jl) at (-1.9,1.3) {$-\!\log|J|$};
    \node[jac] (jr) at (7.9,1.3) {$-\!\log|J|$};

    \node[znode] (z) at (2.4,2.0) {$\vz$};
    \node[font=\scriptsize, right=0.0cm of z, text=red!50!black] {$\sim\!\mathcal{N}(\mathbf{0},\bm{I})$};

    \node[pred] (pred) at (2.4,3.4) {$f_\mathcal{P}(\vu_t, \vz, \vy)$};

    \node[inpt] (y) at (-1, 3.3) {$\vy$};

    \node[dist] (D) at (6,3.4) {$D(\vu_s,\, \hat{\vu}_s)$};

    \draw[arr] (xt) -- (el);
    \draw[arr] (xs) -- (er);

    \draw[arr] (el.west) -- (jl.east);
    \draw[arr] (er.east) -- (jr.west);

    \draw[arr] (el.north) -- (0,2.7) node[left, font=\small]{$\vu_t$} to[out=90,in=190] (pred.west);

    \draw[rarr] (z) -- (pred);

    \draw[arr] (y.north east) to[out=30,in=160] (pred.north west);

    \draw[garr] (pred.east) -- node[above,font=\scriptsize]{$\hat{\vu}_s$} (D.west);

    \draw[arr] (er.north) -- (6,2.7) node[right, font=\small]{$\vu_s$} -- (D.south);

    \draw[dashed, gray!40, line width=0.5pt] (el.east) -- node[below,font=\tiny\sffamily\color{gray}]{shared} (er.west);

    \end{tikzpicture}
\caption{\textbf{NTM overview.}
    Shared transporter $f_\mathcal{T}$ maps $\vx_t, \vx_s$ to representations $\vu_t, \vu_s$ with a tractable Jacobian.
    The predictor $f_\mathcal{P}$ takes $\vu_t$ and latent $\vz \sim \mathcal{N}(\mathbf{0}, \bm{I})$ to produce $\hat{\vu}_s$.
    $D$ measures the distance between the prediction and the target at distribution level.}
    \label{fig:overview}
\end{figure}

\subsection{Model Formulation}
\label{sec:formulation}

As discussed in \cref{sec:prelim_limitation}, modeling $p(\vx_s \mid \vx_t)$ with a Gaussian formulation is fundamentally limited: the true reverse conditional is generally non-Gaussian because it marginalizes over all clean images consistent with $\vx_t$.
We seek a more expressive family that provides exact likelihood for stable training, while remaining structurally close to the diffusion framework to preserve its scalability.

\methodname{} models $p(\vx_s \mid \vx_t)$ by learning to predict in a latent space where the conditional distribution is simple enough to be modeled by Gaussian.
As shown in \cref{fig:overview}, a shared transporter $f_\mathcal{T}$ maps both $\vx_s$ and $\vx_t$ to a latent u-space, and a stochastic predictor $f_\mathcal{P}$ generates $\hat{\vu}_s$ from the noisier representation $\vu_t$ and a latent variable $\vz \sim \mathcal{N}(\mathbf{0}, \mI)$, optionally conditioned on $\vy$ (e.g., text or class label).
\begin{equation}
\label{eq:ntm_coupling}
    \hat{\vu}_s = f_\mathcal{P}(\vu_t, \vz, \vy), \quad \vu_s = f_\mathcal{T}(\vx_s, s), \quad \vu_t = f_\mathcal{T}(\vx_t, t).
\end{equation}
The general training objective minimizes a distributional distance $D$ between the prediction and the target, regularized by $R(f_\mathcal{T})$ to prevent representation collapse~\citep{grill2020bootstrap}:
\begin{equation}
\label{eq:ntm_general}
    \mathcal{L} = \mathbb{E}_{\vz}\big[D(\vu_s,\, \hat{\vu}_s)\big] + R(f_\mathcal{T}).
\end{equation}
Such objectives are common in self-supervised representation learning~\citep{grill2020bootstrap,caron2021emerging,bardes2022vicreg,assran2023self}, but are generally difficult to cast within a probabilistic framework for generative modeling.
The key insight of \methodname{} is that making $f_\mathcal{T}$ an \emph{invertible}, same-dimensional transporter---rather than a compressive encoder---turns this representation-learning objective into exact log-likelihood optimization via the change-of-variables formula.

Specifically, we implement $f_\mathcal{T}$ as a stack of TarFlow blocks~\citep{zhainormalizing,gu2025starflow} with spatial NVP coupling (\cref{eq:nvp}), and $f_\mathcal{P}$ as an affine map $\hat{\vu}_s = \bm{\mu}_\mathcal{P}(\vu_t, t, s, \vy) + \bm{\sigma}_\mathcal{P}(\vu_t, t, s, \vy) \cdot \vz$, which defines $p_\mathcal{P}(\vu_s \mid \vu_t, \vy) = \mathcal{N}(\bm{\mu}_\mathcal{P},\, \mathrm{diag}(\bm{\sigma}_\mathcal{P}^2))$.
Under these choices, setting $D = -\log p_\mathcal{P}$ and $R = -\log|\det J_{f_\mathcal{T}}|$ recovers the exact negative log-likelihood of $p(\vx_s \mid \vx_t)$:
\begin{equation}
\label{eq:ntm_nll}
    \mathcal{L}_{\text{\methodname{}}} = -\log p(\vx_s \mid \vx_t) = -\log p_\mathcal{P}(\vu_s \mid \vu_t) - \log\big|\det J_{f_\mathcal{T}}\big|.
\end{equation}
The composed mapping $\vx_s \xleftrightarrow{f_\mathcal{T}} \vu_s \xleftrightarrow{f_\mathcal{P}} \vz$ forms a \textbf{normalizing flow} from $\vx_s$ to $\vz \sim \mathcal{N}(\mathbf{0}, \mI)$. 
By expanding over a trajectory of $T$ steps, the \methodname{} loss can be simplified as:
\begin{equation}
\label{eq:ntm_loss}
    \mathcal{L}_{\text{\methodname{}}} = \sum_{k=1}^{T} \Big[\tfrac{1}{2}\|\vz_k\|^2 + \sum_n \Big(\log \bm{\sigma}_\mathcal{P}^{(k,n)} + \sum_{\ell} \log \bm{\sigma}_\mathcal{T}^{(k,\ell,n)}\Big) \Big],
\end{equation}
where $\bm{\sigma}_\mathcal{P}^{(k,n)}$ is the predictor scale at step $k$ and position $n$, and $\bm{\sigma}_\mathcal{T}^{(k,\ell,n)}$ is the scale from transporter block $\ell$.
This is the exact negative log-likelihood of the trajectory and training minimizes it end-to-end.

\subsection{Training from Scratch}
\label{sec:from_scratch}


\paragraph{Architecture.}
\methodname{} adopts the deep-shallow architecture of STARFlow~\citep{gu2025starflow,gu2025starflowv}, with a key modification to the deep block.
The \emph{predictor} ($f_\mathcal{P}$) is a deep Transformer that replaces STARFlow's spatial autoregressive flow with a non-causal full-attention coupling layer operating over the \emph{trajectory} dimension.
It predicts $\bm{\mu}_\mathcal{P}(\vu_t, t, s, \vy)$ and $\bm{\sigma}_\mathcal{P}(\vu_t, t, s, \vy)$ for each denoising step.
Despite its depth, the predictor processes all spatial positions in parallel, making it efficient at inference.
The \emph{transporter} ($f_\mathcal{T}$) consists of a few shallow TarFlow-style~\citep{zhainormalizing} causal autoregressive flow blocks with alternating scan directions.
Although autoregressive by nature, each transporter block is lightweight and operates locally within a single denoising step without information leakage across timestep.

\paragraph{Training.}
Given a $T$-step schedule $t_\mathrm{min} = t_0 < t_1 < \cdots < t_T = 1$, we model the joint trajectory distribution as:
\begin{equation*}
\label{eq:traj_factor}
    p(\vx_{t_T}, \ldots, \vx_{t_0}) = p(\vx_{t_T}) \prod_{k=1}^{T} p(\vx_{t_{k-1}} \mid \vx_{t_k}).
\end{equation*}
Since $t_T = 1$ is pure noise, we fix $p(\vx_{t_T}) = \mathcal{N}(\mathbf{0}, \mI)$ and skip both $f_\mathcal{T}$ and $f_\mathcal{P}$ at this level, so the model only learns the conditional factors $p(\vx_s \mid \vx_t)$.
Given clean data $\vx_0$, we construct a stochastic forward trajectory via \cref{eq:markov_transition} and train with either:
\begin{itemize}[leftmargin=*,nosep]
    \item \emph{End-to-end}: compute the \methodname{} loss (\cref{eq:ntm_loss}) over all $T$ conditional factors in the trajectory.
    \item \emph{Pair-wise}: randomly sample a single consecutive pair $(t, s)$ with $s < t$ per batch element.
\end{itemize}
In both modes, each batch element independently samples $T$ from a predefined set (e.g., $\{4, 8, 16\}$), enabling a single model to generate with different step counts without retraining.
For such cases, $f_\mathcal{T}$ takes $T$ as an additional input to adapt to the local timestep spacing.

\paragraph{Sampling.}
Given a schedule $t_\mathrm{min} = s_0 < s_1 < \cdots < s_T \approx 1$, sampling proceeds from noise to data by inverting \cref{eq:ntm_coupling}: the predictor runs sequentially over $T$ steps, drawing $\vz \sim \mathcal{N}(\mathbf{0}, \mI)$ and computing $\hat{\vu}_s = \bm{\mu}_\mathcal{P}(\hat{\vu}_t, t, s) + \bm{\sigma}_\mathcal{P}(\hat{\vu}_t, t, s) \cdot \vz$ at each step, where each output feeds into the next.
After all $T$ predictor steps, the transporter inverts the spatial mapping $\hat{\vx}_0 = f_\mathcal{T}^{-1}(\hat{\vu}_0)$ via sequential AR decoding to produce the final sample in x-space.
Classifier-free guidance~\citep{ho2022classifier} is applied by interpolating the predictor's conditional and unconditional outputs~\citep{gu2025starflow}.

\paragraph{Trajectory Score Denoising.}
Normalizing flows require data to be dense for likelihood training, while natural images often lie on low-dimensional manifolds; TarFlow addresses this by adding a small noise and applying score-based denoising at test time~\citep{zhainormalizing,gu2025starflow}.
In \methodname{}, this extends naturally: the generated trajectory $\hat{\vx} = (\hat{\vx}_{t_0}, \ldots, \hat{\vx}_{t_T})$ is inherently a noisy sequence from the Markov forward process, requiring no additional noise injection.
However, unlike independent per-sample denoising, the trajectory elements are correlated across timesteps.
The \methodname{} loss provides $-\log p(\hat{\vx})$, whose gradient gives the \emph{joint} score of the full trajectory distribution.
We exploit this to perform trajectory-level denoising:
\begin{equation}
\label{eq:refinement}
    \hat{\vx}^{\mathrm{den}} = \frac{1}{1 - \vt}\left(\hat{\vx} - \bm{S} \cdot \nabla_{\hat{\vx}} \mathcal{L}_{\text{\methodname{}}}\right),
\end{equation}
where $\bm{S}$ is the covariance matrix of the trajectory under the pre-defined forward process (\Cref{eq:markov_transition}), with $[\bm{S}]_{ij} = \min(t_i, t_j)^2 (1 - \max(t_i, t_j)) / (1 - \min(t_i, t_j))$, and division by $(1{-}\vt)$ maps from the noisy domain to the clean domain.
The final output is taken at $t_\mathrm{min}$.

\subsection{Finetuning from Pretrained Models}
\label{sec:from_pretrained}

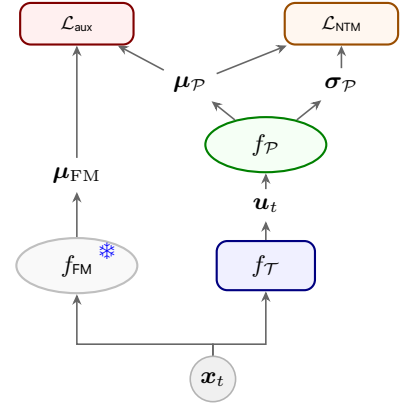
\begin{wrapfigure}{r}{0.42\textwidth}
\vspace{-12pt}
\centering
\begin{tikzpicture}[
    scale=1, every node/.style={scale=1},
    enc/.style={rectangle, rounded corners=4pt, minimum width=1.3cm, minimum height=0.65cm,
        draw=blue!50!black, line width=0.8pt, fill=blue!6, font=\small\sffamily},
    pred/.style={ellipse, minimum width=1.6cm, minimum height=0.7cm,
        draw=green!50!black, line width=0.8pt, fill=green!4, font=\small\sffamily},
    dist/.style={rectangle, rounded corners=4pt, minimum width=1.5cm, minimum height=0.55cm,
        draw=red!50!black, line width=0.8pt, fill=red!6, font=\scriptsize\sffamily},
    inpt/.style={circle, minimum size=0.6cm, draw=gray!60, line width=0.6pt, fill=gray!12, font=\small, inner sep=1pt},
    arr/.style={-{Stealth[length=4pt,width=3.5pt]}, line width=0.6pt, gray!80!black},
]
    \node[inpt] (xt) at (1.8, 0) {$\vx_t$};

    \node[pred, draw=gray!50, fill=gray!5] (fm) at (0, 1.5) {$f_\text{FM}$};
    \node[font=\small, blue] at (0.4, 1.7) {\ding{100}};
    \node[font=\small] (mufm) at (0, 2.7) {$\bm{\mu}_\text{FM}$};

    \node[enc] (g) at (2.5, 1.5) {$f_\mathcal{T}$};
    \node[font=\small] (ut) at (2.5, 2.3) {$\vu_t$};
    \node[pred] (fp) at (2.5, 3.1) {$f_\mathcal{P}$};
    \node[font=\small] (mu) at (1.5, 3.9) {$\bm{\mu}_\mathcal{P}$};
    \node[font=\small] (sig) at (3.5, 3.9) {$\bm{\sigma}_\mathcal{P}$};

    \node[dist] (laux) at (0, 4.7) {$\mathcal{L}_\text{aux}$};
    \node[dist, draw=orange!60!black, fill=orange!6] (lnf) at (3.5, 4.7) {$\mathcal{L}_\text{NTM}$};

    \draw[arr] (xt.north) -- ++(0,0.15) -| (fm.south);
    \draw[arr] (xt.north) -- ++(0,0.15) -| (g.south);

    \draw[arr] (g) -- (ut);
    \draw[arr] (ut) -- (fp);
    \draw[arr] (fp) -- (mu);
    \draw[arr] (fp) -- (sig);

    \draw[arr] (fm) -- (mufm);

    \draw[arr] (mu) -- (laux);
    \draw[arr] (mufm) -- (laux);

    \draw[arr] (mu) -- (lnf);
    \draw[arr] (sig) -- (lnf);

\end{tikzpicture}
\vspace{-6pt}
\caption{Finetuning: $\mathcal{L}_\text{aux}$ aligns $\bm{\mu}_\mathcal{P}$ with frozen $\bm{\mu}_\text{FM}$; $\mathcal{L}_\text{NTM}$ trains the full model.}
\label{fig:aux_loss}
\end{wrapfigure}

\methodname{} can also be initialized from a pretrained flow matching or diffusion models.
Taking flow matching as an example, the pretrained backbone is trained to predict the velocity field in x-space given noisy input $\vx_t$ and timestep $t$.
Here, we reinterpret the prediction $\vv$ and hidden states $\vh$ from the input $(\vu_t, t)$ in u-space.
We can readily compute a predicted clean sample $\hat{\vu}_0 = \vu_t - t \cdot \vv$ and derive the denoising posterior $\mathcal{N}(\bm{\mu}_\mathrm{post}, \sigma_\mathrm{post}^2 \mI)$ for the transition from $t$ to $s$:
\begin{equation}
\label{eq:posterior}
    \bm{\mu}_\mathrm{post} = A(t,s) \, \vu_t + B(t,s) \, \hat{\vu}_0, \qquad
    \sigma_\mathrm{post} = C(t,s),
\end{equation}
where $A$, $B$, and $C$ are closed-form coefficients derived from the true reverse posterior of the Markovian forward process (\cref{sec:prelim_traj}; full derivation in \cref{sec:app_posterior}).
We initialize the predictor to match this posterior: $\bm{\mu}_\mathcal{P} = \bm{\mu}_\mathrm{post}$, and learn a multiplicative scale correction via a zero-initialized projection:
\begin{equation}
\label{eq:residual_coupling}
    \bm{\mu}_\mathcal{P} = \bm{\mu}_\mathrm{post}, \qquad
    \bm{\sigma}_\mathcal{P} = \sigma_\mathrm{post} \cdot \exp(\bm{\delta}_\sigma), \quad \bm{\delta}_\sigma = \mathrm{proj}_\mathrm{out}(\vh),
\end{equation}
where $\mathrm{proj}_\mathrm{out}$ is initialized to zero so that $\bm{\sigma}_\mathcal{P} = \sigma_\mathrm{post}$ at initialization.
By further initializing the transporter as identity ($f_\mathcal{T} = \mathrm{id}$), the full model starts as the pretrained Gaussian posterior in x-space.
As training progresses, the NLL objective drives $f_\mathcal{T}$ to drift from identity and $\bm{\delta}_\sigma$ to depart from zero, jointly learning the non-Gaussian structure of $p(\vx_s \mid \vx_t)$.

\paragraph{Mean-alignment auxiliary loss.}
To prevent early divergence from the pretrained solution, we add an auxiliary loss that aligns \methodname{}'s learned shift $\bm{\mu}_\mathcal{P}$ with the denoising mean $\bm{\mu}_\mathrm{FM}$ produced by a frozen copy of the pretrained backbone predicting directly from x-space:
\begin{equation}
\label{eq:fm_aux}
    \mathcal{L}_\mathrm{aux} = \big\| \bm{\mu}_\mathcal{P} - \bm{\mu}_\mathrm{FM} \big\|^2.
\end{equation}
The total loss is $\mathcal{L} = \mathcal{L}_{\text{\methodname{}}} + \lambda  \; \mathcal{L}_\mathrm{aux}$, where $\lambda$ can be annealed during training.
This auxiliary loss serves three purposes:
(1) it encourages the model to remain close to the pretrained diffusion solution, preventing catastrophic drift;
(2) $\bm{\mu}_\mathrm{FM}$ itself defines a meaningful u-space---since it is a neural prediction of the next-step mean directly from $\vx_t$, it is smooth and predictable, and $\mathcal{L}_\mathrm{aux}$ ensures the transporter learns to connect these per-step predictions into a coherent trajectory;
(3) because the transporter and predictor can move jointly, the model can optimize the NF loss without drifting from the pretrained quality.

\subsection{Fast Generation via Learned Denoiser}
\label{sec:denoiser}
\begin{figure*}[t]
    \centering
    \begin{tikzpicture}[
        scale=0.8, every node/.style={scale=0.8},
        enc/.style={rectangle, rounded corners=4pt, minimum width=1.3cm, minimum height=0.7cm,
            draw=blue!50!black, line width=0.8pt, fill=blue!6, font=\small\sffamily},
        pred/.style={ellipse, minimum width=1.6cm, minimum height=0.7cm,
            draw=green!50!black, line width=0.8pt, fill=green!4, font=\small\sffamily},
        den/.style={ellipse, minimum width=1.6cm, minimum height=0.7cm,
            draw=teal!70!black, line width=1pt, fill=teal!10, font=\small\sffamily},
        arr/.style={{Stealth[length=4pt,width=3.5pt]}-{Stealth[length=4pt,width=3.5pt]}, line width=0.6pt, gray!80!black},
        uparr/.style={-{Stealth[length=4pt,width=3.5pt]}, line width=0.6pt, gray!80!black},
        darr/.style={-{Stealth[length=5pt,width=4pt]}, line width=0.8pt, olive!80!black},
        gradarr/.style={-{Stealth[length=4pt,width=3.5pt]}, line width=0.7pt, orange!70!black, dashed},
    ]

    \def\sep{2.2}

    \fill[red!4, rounded corners=6pt] (-0.9, -1.1) rectangle (3.5*2.2 + 0.9, 3.7);
    \draw[orange!70!black, dashed, line width=0.8pt, rounded corners=6pt]
        (-0.9, -1.1) rectangle (3.5*2.2 + 0.9, 3.7);

    \node[font=\small, orange!70!black] (xdT) at (0, -2.0) {$\vx^{\text{den}}_{t_T}$};
    \node[font=\small, orange!70!black] (xdT1) at (\sep, -2.0) {$\vx^{\text{den}}_{t_{T\!-\!1}}$};
    \node[font=\small, orange!70!black] (xd2) at (2.5*\sep, -2.0) {$\vx^{\text{den}}_{t_1}$};
    \node[font=\small, orange!70!black] (xd1) at (3.5*\sep, -2.0) {$\vx^{\text{den}}_{t_0}$};

    \node[font=\small] (xT) at (0, -0.7) {$\vx_{t_T}$};
    \node[font=\small] (xT1) at (\sep, -0.7) {$\vx_{t_{T\!-\!1}}$};
    \node[font=\small] (dots_x) at (1.75*\sep, -0.7) {$\cdots$};
    \node[font=\small] (x2) at (2.5*\sep, -0.7) {$\vx_{t_1}$};
    \node[font=\small] (x1) at (3.5*\sep, -0.7) {$\vx_{t_0}$};

    \node[font=\small] (uT) at (0, 1.2) {$\vu_{t_T}$};
    \node[font=\small] (zT) at (0, 3.2) {$\vz_{t_T}$};
    \draw[arr] (xT) -- node[right, font=\tiny, gray]{$=$} (uT);
    \draw[arr] (uT) -- node[right, font=\tiny, gray]{$=$} (zT);

    \node[enc] (gT1) at (\sep, 0.3) {$f_\mathcal{T}$};
    \node[enc] (g2) at (2.5*\sep, 0.3) {$f_\mathcal{T}$};
    \node[enc] (g1) at (3.5*\sep, 0.3) {$f_\mathcal{T}$};

    \node[font=\small] (uT1) at (\sep, 1.2) {$\vu_{t_{T\!-\!1}}$};
    \node[font=\small] (u2) at (2.5*\sep, 1.2) {$\vu_{t_1}$};
    \node[font=\small] (u1) at (3.5*\sep, 1.2) {$\vu_{t_0}$};

    \node[pred] (fT1) at (\sep, 2.2) {$f_\mathcal{P}$};
    \node[font=\small] (dots_f) at (1.75*\sep, 2.2) {$\cdots$};
    \node[pred] (f2) at (2.5*\sep, 2.2) {$f_\mathcal{P}$};
    \node[pred] (f1) at (3.5*\sep, 2.2) {$f_\mathcal{P}$};

    \node[font=\small] (zT1) at (\sep, 3.2) {$\vz_{t_{T\!-\!1}}$};
    \node[font=\small] (z2) at (2.5*\sep, 3.2) {$\vz_{t_1}$};
    \node[font=\small] (z1) at (3.5*\sep, 3.2) {$\vz_{t_0}$};

    \node[den] (phi) at (4.7*\sep, 2.2) {$g_\phi$};

    \node[rectangle, rounded corners=5pt, minimum width=1.5cm, minimum height=0.5cm,
        draw=red!50!black, line width=1pt, fill=red!6, font=\small\sffamily]
        (loss) at (4.7*\sep, -2.0) {$\mathcal{L}_\text{den}$};

    \draw[arr] (xT1) -- (gT1);
    \draw[arr] (x2) -- (g2);
    \draw[arr] (x1) -- (g1);

    \draw[arr] (gT1) -- (uT1);
    \draw[arr] (g2) -- (u2);
    \draw[arr] (g1) -- (u1);

    \draw[arr] (uT1) -- (fT1);
    \draw[arr] (u2) -- (f2);
    \draw[arr] (u1) -- (f1);

    \draw[arr] (fT1) -- (zT1);
    \draw[arr] (f2) -- (z2);
    \draw[arr] (f1) -- (z1);

    \draw[uparr] (uT.east) to[out=0,in=180] (fT1.west);
    \draw[uparr] (u2.east) to[out=0,in=180] (f1.west);

    \draw[darr] (u1.east) to[out=0,in=180] (phi.west);
    \draw[darr] (phi.south) -- (loss.north);

    \draw[orange!70!black, -{Stealth[length=5pt,width=4pt]}, line width=0.7pt]
        (xd1.east) -- (loss.west);

    \draw[gradarr] (xT.south) -- (xdT.north);
    \draw[gradarr] (xT1.south) -- (xdT1.north);
    \draw[gradarr] (x2.south) -- (xd2.north);
    \draw[gradarr] (x1.south) -- (xd1.north);

    \node[font=\scriptsize, orange!70!black] at (1.75*\sep, -1.35) {$-\bm{S}\!\cdot\!\nabla_{\vx}\mathcal{L}_\text{NTM}$};

    \end{tikzpicture}
    \caption{\textbf{Denoiser training via trajectory score denoising.}
    The frozen NTM (dashed box) computes the trajectory NLL and its gradient refines every position via $\vx^{\text{den}} = \vx - \bm{\Sigma}\nabla_\vx \mathcal{L}_\text{NTM}$, producing denoised targets (orange).
    A denoiser $g_\phi$ learns to predict $\vx^{\text{den}}_{t_0}$ directly from the $\vu_{t_0}$.}
    \label{fig:denoiser}
    \vspace{-10pt}
\end{figure*}
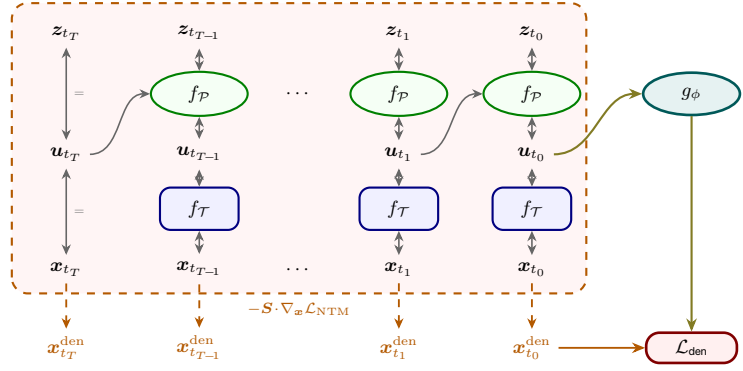

Standard sampling from \methodname{} requires $T$ sequential predictor steps with AR decoding at each step, together with the trajectory score-based denoising (\cref{eq:refinement}) using backpropagation at test time.
Both of them, while acceptable due to the light-weight design, still introduce more latency than the predictor.
To eliminate this cost, we can optionally train a lightweight denoiser network $g_\phi$ that amortizes the self-refinement into a single forward pass, following a similar distillation paradigm of NFM~\citep{berthelot2026nfm} and STARFlow-V~\citep{gu2025starflowv}.
The denoiser is a Transformer with non-causal attention that takes the predictor's output $\vu_{t_0}$ at the cleanest level in u-space along with text embeddings $\vy$, and directly outputs a denoised image $\hat{\vx}_0^{\mathrm{den}}$.
Since we model a Markov trajectory and the designed invertibility, $\vu_{t_0}$ already contains all the information needed to \textbf{deterministically} predict the clean output.

The denoiser can be post-trained after the main model converges, using MSE against score-based denoising targets derived from the frozen \methodname{} model on real data trajectories (\Cref{eq:refinement}):
\begin{equation}
\label{eq:denoiser_loss}
    \mathcal{L}_\mathrm{den} = \big\| g_\phi(\vu_{t_0}, \vy) - \hat{\vx}_0^{\mathrm{den}} \big\|^2.
\end{equation}
At inference, the new pipeline becomes: (1) run the predictor over $T$ steps to produce $\vu_{t_0}$, (2) run $g_\phi$ in a single forward pass to obtain $\hat{\vx}_0$. This bypasses both the transporter AR decoding and the backprop-based denoising, producing high-quality images in as few as \textbf{four steps}.

\section{Experiments}
\label{sec:experiments}

\subsection{Setup}
\label{sec:exp_setup}

\paragraph{Implementation.}
All \methodname{} models are trained with AdamW in bfloat16 with FSDP on an internal text-image dataset of $\sim$70M pairs (including CC12M).
We consider two settings:
\begin{itemize}[leftmargin=*,nosep]
    \item \emph{From scratch}: class-conditional ImageNet and text-to-image generation at $256{\times}256$ resolution with the latent space of FAE~\citep{gao2025one} ($16\times$ spatial compression, 32-dim latents), using Qwen-2.5-VL as the text encoder.
    \item \emph{Finetuning}: initializing from a pretrained flow-matching backbone (FLUX.2-klein, 4B)\footnote{\url{https://huggingface.co/black-forest-labs/FLUX.2-klein-base-4B}} at $512{\times}512$ resolution with its native VAE latent space.
\end{itemize}

The transporter consists of 2 TarFlow-style blocks with 4 layers each and causal masks along alternating directions; the predictor is a 24-layer full-attention Transformer.
All models use $T{=}4$ denoising steps and 10\% CFG dropout.
For finetuning, we apply the residual parameterization (\cref{sec:from_pretrained}) with the auxiliary loss ($\lambda{=}2.5$, MSE variant).
Both settings use a batch size of $1024$ on 64 H100 GPUs.
Further details are in the Appendix.

\paragraph{Evaluation.}
We report compositional accuracy on GenEval~\citep{ghosh2023geneval} and DPG-Bench~\citep{hu2024ella} for text-to-image generation.
We additionally evaluate class-conditional generation on ImageNet 256$\times$256 for fair comparison when training \methodname{} from scratch (\cref{sec:app_imagenet}).






\begin{table}[t]
\centering
\small
\setlength{\tabcolsep}{4pt}
\caption{\textbf{T2I Evaluation.}
GenEval~\citep{ghosh2023geneval} overall score and DPG-Bench~\citep{hu2024ella} percentage. }
\label{tab:geneval_dpg}
\begin{tabular}{llcc}
\toprule
Type & Method & GenEval$\uparrow$ & DPG$\uparrow$ \\
\midrule
\multirow{9}{*}{DM} & SDXL~\citep{podell2023sdxl} & 0.55 & 74.65 \\
& PixArt-$\alpha$~\citep{chen2023pixartalpha} & 0.48 & 71.11 \\
& SD3-Medium~\citep{esser2024scaling} & 0.62 & 84.08 \\
& FLUX.1-dev~\citep{labs2024flux} & 0.66 & 83.84 \\
& Janus-Pro-7B~\citep{chen2025januspro} & 0.80 & 84.19 \\
& HiDream-I1-Full~\citep{hidream2025i1} & 0.83 & 85.89 \\
& Seedream 3.0~\citep{seedream2025} & 0.84 & 88.27 \\
& Qwen-Image~\citep{wu2025qwenimage} & 0.87 & 88.32 \\
& Nucleus-Image~\citep{akiti2026nucleus} & 0.87 & 88.79 \\
\midrule
\multirow{3}{*}{NF} & STARFlow~\citep{gu2025starflow} & 0.56 & -- \\
& \methodname{} (from scratch, $256\times 256$) & 0.82 & 79.64 \\
& \methodname{} (finetune, $512\times 512$) & 0.76 & 83.38 \\
\bottomrule
\end{tabular}

\end{table}

\subsection{From-Scratch Results}
\label{sec:results_scratch}

\paragraph{Text-to-image generation.}
\cref{tab:geneval_dpg} reports compositional accuracy on GenEval and DPG-Bench.
\methodname{} trained from scratch at $256{\times}256$ achieves \textbf{0.82} on GenEval and \textbf{79.64} on DPG-Bench with only 4 steps, significantly outperforming the prior normalizing flow model STARFlow~\citep{gu2025starflow} (0.56 GenEval, 256 autoregressive steps) and matching strong diffusion baselines that require substantially more sampling steps.

\paragraph{Class-conditional ImageNet.}
As a controlled comparison for the from-scratch setting, we evaluate on class-conditional ImageNet $256{\times}256$.
\methodname{} achieves 2.80 FID with 16 steps and 3.83 FID with 4 steps---comparable to STARFlow (FAE) at 2.67 FID which requires 256 autoregressive steps (\cref{sec:app_imagenet}).
These results use only the exact NLL training objective without any distribution-level losses (e.g., adversarial or perceptual), demonstrating that exact likelihood training alone produces competitive few-step generation.

\subsection{Finetuning Results}
\label{sec:results_finetune}

\paragraph{Text-to-image generation.}
The finetuned variant at $512{\times}512$ achieves 0.76 on GenEval and 83.38 on DPG-Bench (\cref{tab:geneval_dpg}), demonstrating that \methodname{} can scale to higher resolutions via pretrained initialization.
The position and attribute-binding sub-tasks remain challenging at this stage of finetuning, suggesting room for improvement with longer training or stronger pretrained backbones.

\begin{wraptable}{r}{0.45\textwidth}
\vspace{-1.2em}
\centering
\small
\caption{Score denoising vs.\ learned denoiser (finetuned setting).}
\label{tab:denoiser}
\begin{tabular}{lcc}
\toprule
Method & img/s $\uparrow$ & LPIPS $\downarrow$ \\
\midrule
Full NF + Traj. denoise & 0.20 & --- \\
Predictor + Denoiser & 1.88 & 0.121 \\
\bottomrule
\end{tabular}
\vspace{-1em}
\end{wraptable}

\paragraph{Score denoising vs.\ learned denoiser.}
\cref{tab:denoiser} compares two inference strategies for the finetuned model: (i)~transporter inversion followed by trajectory score denoising via \cref{eq:refinement}, and (ii)~the learned denoiser $g_\phi$ that amortizes the refinement into a single forward pass.
The denoiser achieves $\sim\textbf{9}\times$ speedup while maintaining high fidelity to the score-based refinement output (LPIPS $0.121$), confirming that a single forward pass can effectively replace iterative backpropagation-based denoising.

\subsection{Ablation Studies}
\label{sec:ablations}

\begin{figure*}[t]
    \centering
    \includegraphics[width=\linewidth]{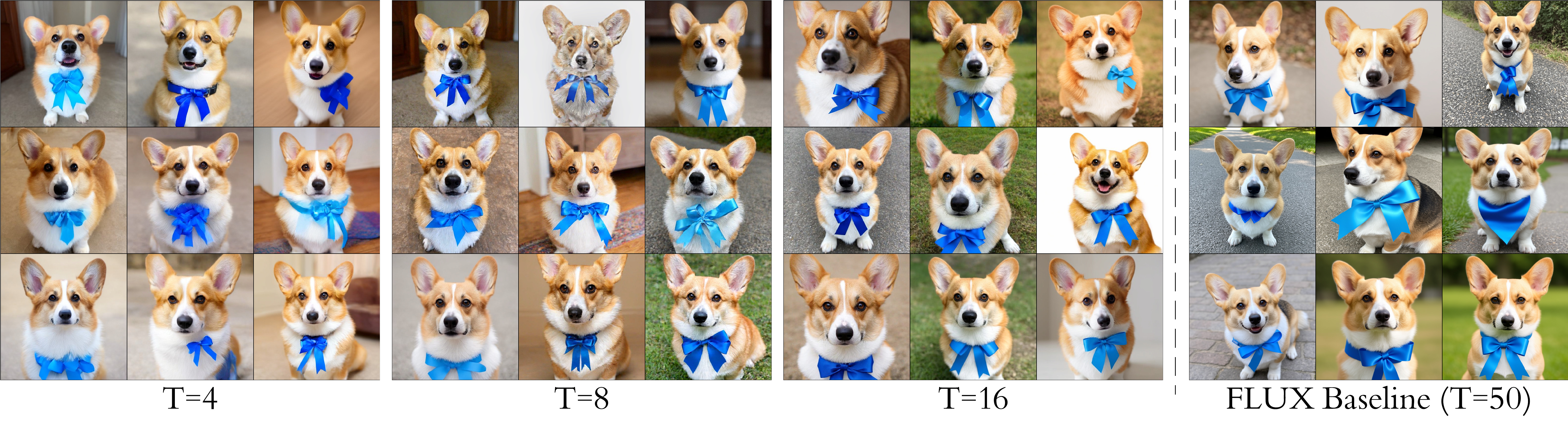}
    \caption{\textbf{Ablation: multi-trajectory training.}
    Comparison of the same \methodname{} evaluated with $T{=}4$, $T{=}8$, $T{=}16$ denoising steps and the baseline FLUX (50 steps).}
    \label{fig:ablation_traj}
\end{figure*}

\begin{figure}[t]
    \centering
    \begin{subfigure}[t]{0.24\linewidth}
        \centering
        \includegraphics[width=\linewidth]{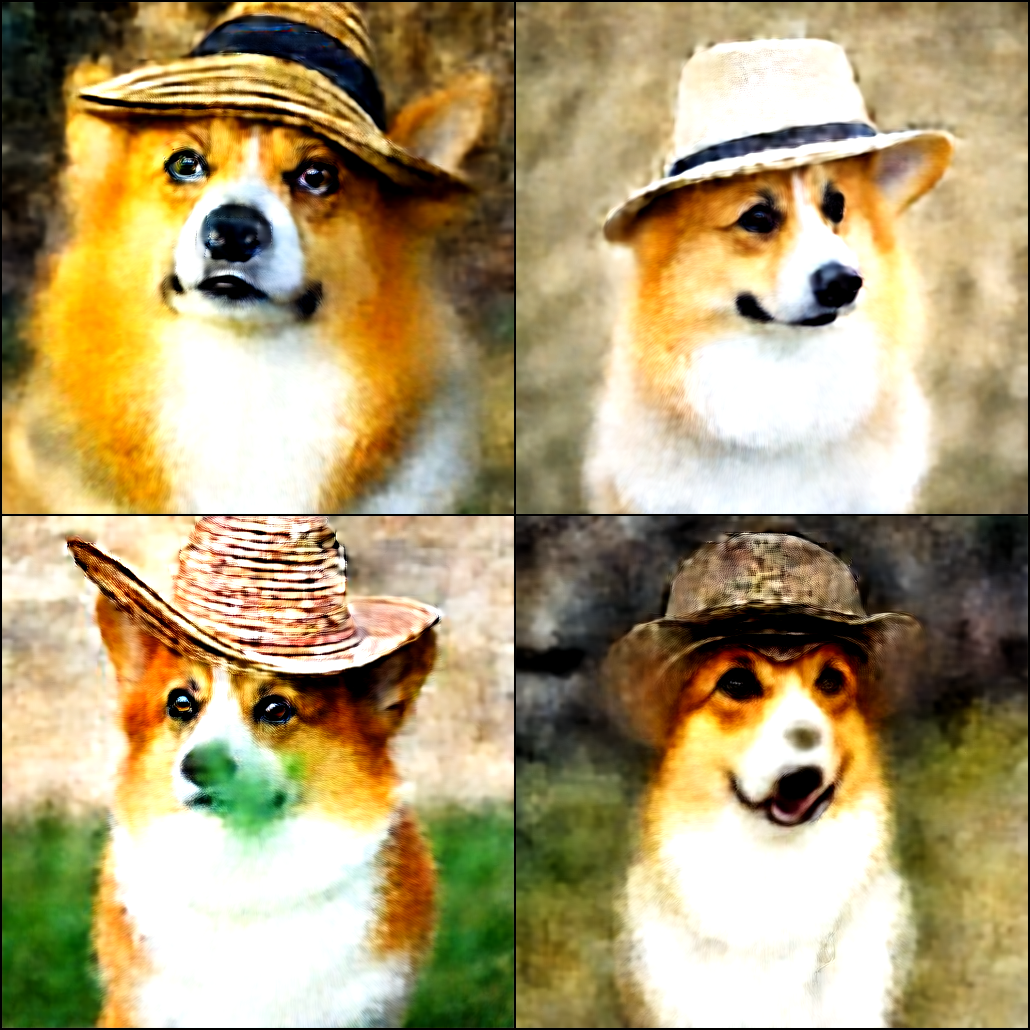}
        \caption{Without aux loss}
    \end{subfigure}
    \hfill
    \begin{subfigure}[t]{0.24\linewidth}
        \centering
        \includegraphics[width=\linewidth]{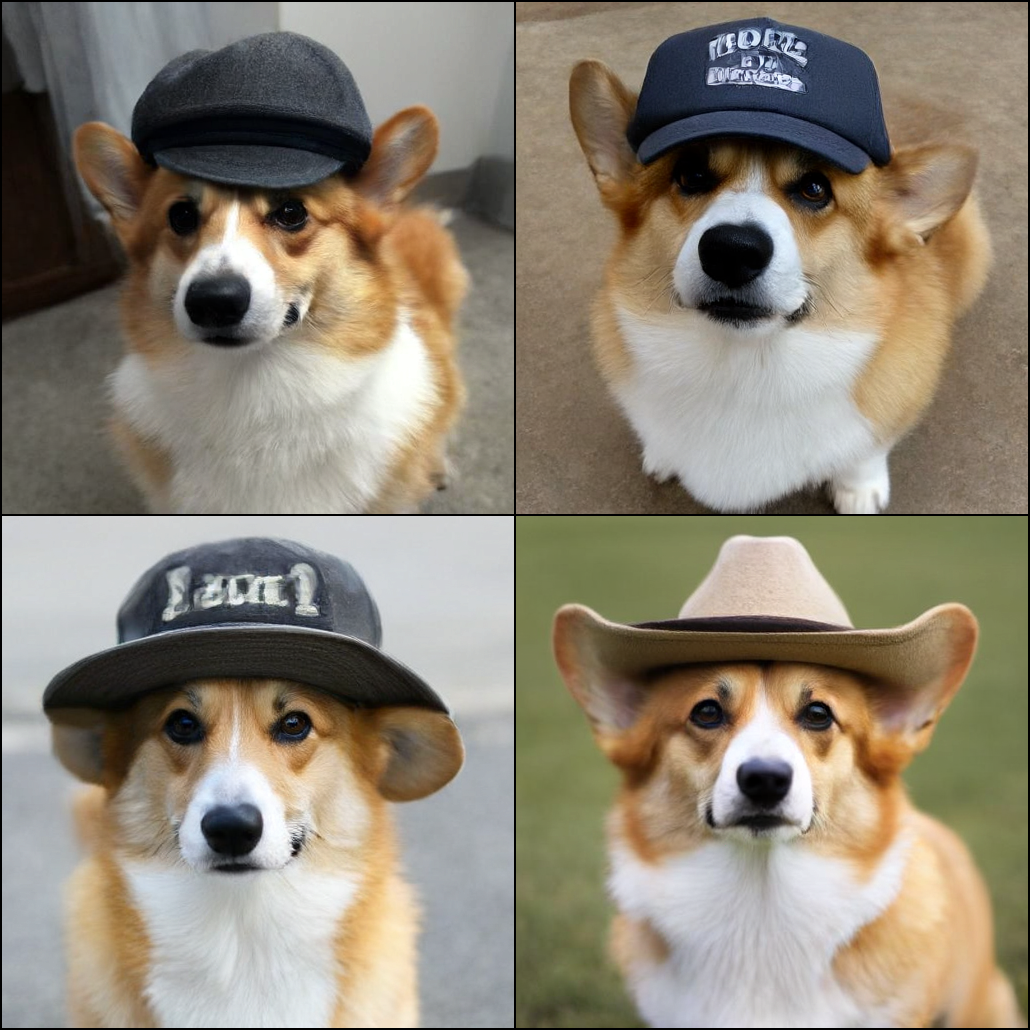}
        \caption{With aux loss}
    \end{subfigure}
    \hfill
    \begin{subfigure}[t]{0.48\linewidth}
        \centering
        \includegraphics[width=\linewidth]{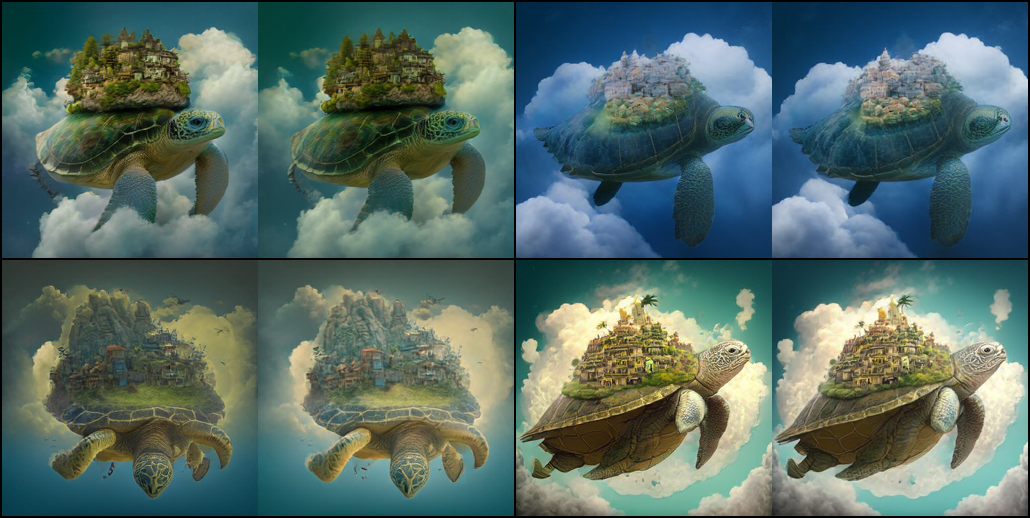}
        \caption{Trajectory score denoising \textit{vs.}\ learned denoiser}
    \end{subfigure}
    \caption{\textbf{Ablations.}
    (a) Finetuning directly with the NF loss diverges.
    (b) Adding the mean-alignment loss (\cref{eq:fm_aux}) stabilizes training.
    (c) Comparison of denoising approaches.}
    \label{fig:ablation_loss}
    \vspace{-10pt}
\end{figure}

We conduct ablation studies on text-to-image generation to analyze the key design choices of \methodname{}.

\paragraph{Multi-trajectory training (finetuned).}
\cref{fig:ablation_traj} compares finetuned models trained with different trajectory lengths $T \in \{4, 8, 16\}$ against the baseline FLUX (50 steps).
Longer trajectories provide finer-grained denoising steps, which can improve detail preservation at the cost of slower inference.
We find that $T{=}4$ provides the best quality--speed trade-off for the finetuning setting.

\paragraph{Effect of the transporter (from scratch).}
As shown in \cref{fig:trajectory_comparison}, reducing flow matching to 4 steps without a transporter produces severely blurry outputs.
The invertible mapping provides a latent space where the affine predictor becomes expressive, recovering 50-step quality in only 4 steps.

\paragraph{Auxiliary loss for finetuning.}
\cref{fig:ablation_loss} ablates the mean-alignment auxiliary loss (\cref{eq:fm_aux}).
Without the auxiliary loss ($\lambda = 0$), finetuning diverges early in training---the NLL objective alone provides insufficient signal to keep the predictor near the pretrained solution, causing catastrophic forgetting.
The auxiliary loss stabilizes training by anchoring the predictor mean to the pretrained velocity field.

\subsection{Qualitative Results}
\label{sec:qualitative}

\cref{fig:teaser} presents text-to-image samples from \methodname{} in 4 denoising steps across both settings.

\paragraph{From scratch ($256{\times}256$).}
The from-scratch model demonstrates strong compositional generalization across multi-object scenes, fine-grained attribute control, and varied artistic styles despite being trained at moderate resolution.

\paragraph{Finetuned ($512{\times}512$).}
The finetuned model preserves the visual quality and prompt adherence of the pretrained FLUX backbone (which requires 50 steps) while operating in only 4 steps, confirming that modeling the non-Gaussian reverse conditional recovers the information lost by naive step reduction.
Samples exhibit high-resolution detail, text rendering capability, and diverse artistic styles.

\begin{figure}[!t]
\centering
\includegraphics[width=\linewidth]{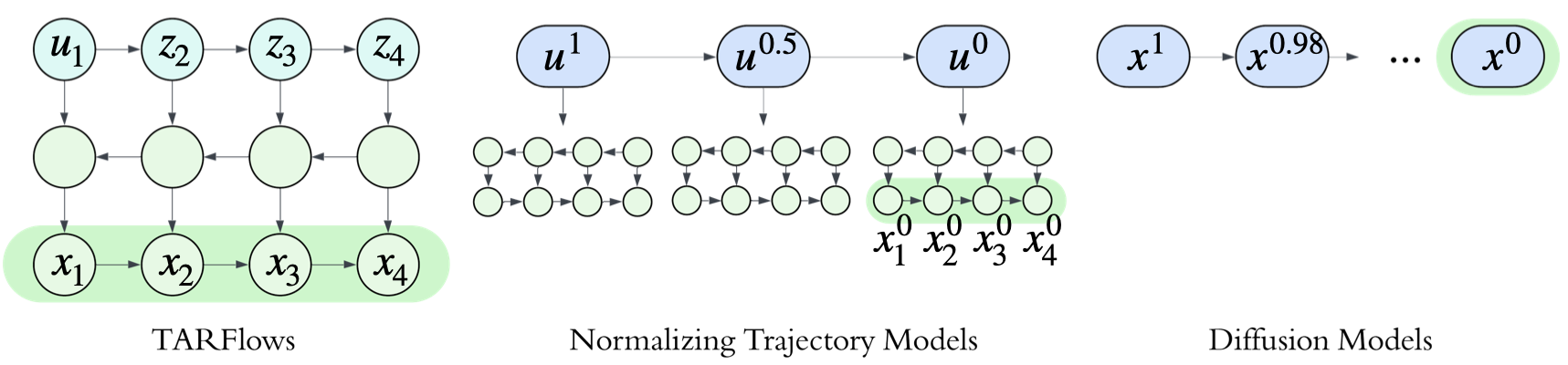}
\caption{\textbf{Comparison between TARFlows, Normalizing Trajectory Models, and Diffusion Models.}
Subscripts denote spatial patch indices, and superscripts denote trajectory timesteps, with $\vx^1, \vu^1$ representing pure Gaussian noise and $\vx^0, \vu^0$ the clean data. 
}
\label{fig:ntm_comp}
\end{figure}

\section{Discussion}
\label{sec:discussion}

\paragraph{NTM as an interpolation between normalizing flows and flow matching.}
STARFlow~\citep{gu2025starflow} directly models the marginal image distribution $p(\vx)$ by decomposing it via a deep spatial autoregressive flow within a single generation step---the entire generation is performed in one pass through many sequential AR blocks (e.g., 256 steps).
At the other extreme, flow matching models a velocity field whose ODE integration requires many small Gaussian steps for high quality.
\methodname{} occupies a middle ground: it explicitly models each intermediate conditional $p(\vx_s \mid \vx_t)$ along a $T$-step denoising trajectory as a normalizing flow, as shown in \cref{fig:ntm_comp}.

\begin{wrapfigure}{r}{0.3\linewidth}
\vspace{-20pt}
\centering
\includegraphics[width=\linewidth]{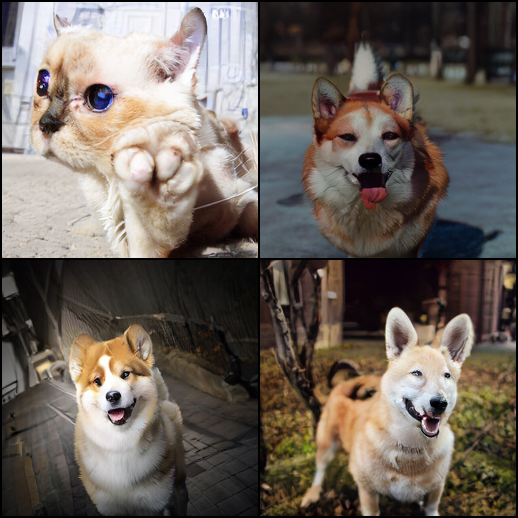}
\caption{\textbf{Failure Case:} \methodname{} with $T{=}1$ produces degraded outputs due to insufficient transporter capacity. Prompt: \textit{a corgi dog.}}
\vspace{-15pt}
\label{fig:failure_1step}
\end{wrapfigure}

The key architectural tradeoff is where to place depth.
STARFlow concentrates all capacity \emph{within} a single step via deep spatial AR blocks; \methodname{} distributes capacity \emph{across} multiple denoising steps, using a shallow transporter (2 blocks $\times$ 4 layers) at each step paired with a deep trajectory-level predictor.
This trades per-step expressiveness for multi-step structure: the predictor reasons across timestep levels while each transporter handles only the local non-Gaussian residual within one step.
As a result, the per-step normalizing flow in \methodname{} can be lightweight because each step only needs to capture the conditional $p(\vx_s \mid \vx_t)$---which is simpler than the full marginal $p(\vx)$---while the deep predictor captures the bulk of the denoising signal in u-space.

\paragraph{Why single-step generation remains challenging.}
As shown in \cref{fig:failure_1step}, \methodname{} with $T{=}1$ produces severely degraded outputs.
This failure is not a training issue but a fundamental capacity constraint.
At $T{=}1$, the entire non-Gaussian structure of the data distribution must be captured by the shallow transporter alone---the predictor reduces to a single-step Gaussian coupling.
This configuration is effectively a STARFlow-like architecture with a parallel (non-causal) prior, but with far fewer transporter layers than STARFlow's deep blocks (8 layers vs.\ STARFlow's 24+ layers per block $\times$ multiple blocks).
Making the transporter as deep as STARFlow would restore single-step quality but defeat the purpose of the few-step design, as inference would again be dominated by sequential AR decoding.

For finetuning, the $T{=}1$ setting introduces additional challenges: the mean-alignment auxiliary loss (\cref{eq:fm_aux}) was designed to anchor the predictor to a multi-step denoising trajectory, and collapsing to a single step fundamentally changes the training dynamics.

\paragraph{Implications.}
\methodname{}'s sweet spot is $T{=}4$--$8$: enough steps for the shallow transporter to distribute the non-Gaussian modeling across the trajectory, while the deep predictor handles cross-timestep reasoning efficiently in parallel.
The architecture naturally admits a spectrum---deeper transporters with fewer steps, or shallower transporters with more steps---offering a principled way to trade off sequential computation for generation quality.
Pushing toward single-step generation with exact likelihood remains an open challenge that may require fundamentally different architectural choices, such as adaptive-depth transporters or progressive capacity allocation across the trajectory.

\section{Related Work}
\label{sec:related}

\paragraph{Normalizing flows for image generation.}
Normalizing flows~\citep{dinh2014nice,dinh2016density,rezende2015variational,kingma2018glow} learn invertible mappings with exact log-likelihood via the change-of-variables formula.
Classical approaches struggled to scale to high-resolution images due to the full-dimensional invertibility constraint.
TarFlow~\citep{zhainormalizing} addressed this by parameterizing autoregressive coupling layers with causal Transformers, enabling flows to leverage modern sequence-modeling architectures.
STARFlow~\citep{gu2025starflow,gu2025starflowv} further introduced a deep-shallow design---a single deep autoregressive block followed by lightweight shallow blocks---scaling normalizing flows to competitive text-to-image generation.
While these methods model the marginal $p(\vx)$ directly, \methodname{} applies normalizing flows to the \emph{conditional} distribution $p(\vx_s \mid \vx_t)$ at each denoising step.
Since conditioning on $\vx_t$ already constrains the space of plausible images, the per-step flow is simpler than the full marginal and requires fewer blocks.

\paragraph{Non-Gaussian reverse processes.}
The Gaussian assumption in diffusion reverse steps has been challenged by several works.
DDGAN~\citep{xiao2022ddgan} trains a GAN discriminator at each denoising step, enabling larger step sizes by modeling an implicit non-Gaussian conditional.
However, GAN-based approaches provide no tractable density, suffer from mode-seeking behavior, and are difficult to scale.
Diffusion Normalizing Flow~\citep{zhang2021diffusion} combines normalizing flows with diffusion via neural SDEs, but models the entire generation trajectory as a single continuous flow rather than learning an expressive per-step reverse conditional.
Concurrent work~\citep{chen2026nfid} also explores normalizing flows with iterative denoising using a different architectural design.
\methodname{} models the non-Gaussian reverse via normalizing flows with exact log-likelihood, providing mode-covering training, stable optimization, and a tractable score.

\paragraph{Few-step generation and distillation.}
Reducing sampling steps is a major research direction.
Progressive distillation~\citep{salimans2022progressive} trains a student to match multi-step teacher outputs in fewer steps.
Consistency models~\citep{song2023consistency} learn to map any point on the trajectory directly to the clean image.
Distribution matching distillation (DMD)~\citep{yin2024improved} and latent consistency models~\citep{luo2023latent} further improve few-step quality via distributional matching objectives.
NFM~\citep{berthelot2026nfm} distills pretrained normalizing flow couplings to train faster flow-matching students.
\methodname{} is complementary to distillation approaches with a learned non-Gaussian reverse and trains a denoiser for fast inference.

\paragraph{Score-based denoising and refinement.}
The connection between denoising and score functions~\citep{song2021scorebased} has been exploited for test-time sample improvement.
TarFlow~\citep{zhainormalizing} introduced adding a small amount of noise and applying the gradient of the NF log-likelihood as a score-based denoiser; STARFlow~\citep{gu2025starflow} extended this to latent-space generation.
These methods perform independent per-sample denoising.
\methodname{} generalizes this to \emph{trajectory-level} denoising: since the generated trajectory is a correlated sequence from the Markov forward process, the NTM loss provides a joint score over all timesteps, and the covariance-weighted gradient correction exploits cross-timestep correlations for more effective refinement than per-sample approaches.

\paragraph{Trajectory-level modeling and flow maps.}
Several methods model generation across multiple trajectory points rather than per-step.
Consistency models~\citep{song2023consistency} and latent consistency models~\citep{luo2023latent} learn to project any noisy point directly to the clean endpoint.
FlowMaps~\citep{flowmaps} generalizes this by learning direct mappings between arbitrary pairs of time points on the probability flow ODE.
Mean flows~\citep{geng2025mean} learn one-step generators via flow matching with mean prediction.
These methods learn deterministic mappings via regression objectives.
\methodname{} is distinct in two ways: it retains a \emph{distributional} model of $p(\vx_s \mid \vx_t)$ (not a point estimate), enabling sampling diversity and likelihood evaluation; and it models the conditional at each step as a normalizing flow rather than collapsing all steps into a single mapping.

\section{Conclusion}
\label{sec:conclusion}

We introduced \fullname{} (\methodname{}), a framework that models each reverse conditional as a normalizing flow via an invertible transporter and a Gaussian predictor, yielding exact log-likelihood training.
\methodname{} supports training from scratch and finetuning from pretrained models, and its trajectory likelihood enables score-based denoising distillable into a four-step sampler.
On text-to-image benchmarks, \methodname{} significantly outperforms prior normalizing flow models and matches strong diffusion baselines with only 4 steps, while uniquely retaining exact likelihood over the generative trajectory.
Future work includes distribution-level post-training (e.g., adversarial or perceptual losses) to further boost few-step quality, scaling to higher resolutions, and exploring architectural designs that push exact-likelihood generation toward even fewer steps.

\bibliographystyle{plainnat}
\bibliography{main}

\appendix

\section{Theoretical Analysis}
\label{sec:app_theory}

\subsection{\methodname{} as a Conditional Normalizing Flow}
\label{sec:app_nf}

We establish that each reverse transition in \methodname{} defines a valid conditional normalizing flow with exact log-likelihood.
Consider a single denoising step from timestep $t$ to $s$ ($s < t$).
\methodname{} maps the clean-side sample $\vx_s$ to a latent $\vz$ via the composition of two invertible transformations:
\begin{enumerate}
    \item \textbf{Transporter} (spatial autoregressive flow): $\vu_s = f_\mathcal{T}(\vx_s)$, an invertible mapping from x-space to u-space via a stack of $L$ TarFlow-style~\citep{zhainormalizing} causal AR blocks with alternating scan directions.
    Each block $\ell$ applies an elementwise affine coupling
    $\vu^{(\ell)}_n = (\vu^{(\ell-1)}_n - \bm{\mu}^{(\ell)}(\vu^{(\ell-1)}_{<n})) / \bm{\sigma}^{(\ell)}(\vu^{(\ell-1)}_{<n}))$
    with a triangular Jacobian.
    \item \textbf{Predictor} (trajectory-level affine coupling): $\vz = (\vu_s - \bm{\mu}_\mathcal{P}(\vu_t, t, s)) / \bm{\sigma}_\mathcal{P}(\vu_t, t, s)$, an affine coupling in u-space conditioned on the noisier representation $\vu_t$.
\end{enumerate}
The composition $\vz = f_\mathcal{P}(f_\mathcal{T}(\vx_s); \vu_t, t, s)$ is invertible (both components are), and the exact log-likelihood follows from the change-of-variables formula:
\begin{equation}
\label{eq:app_nll_decomposition}
    \log p(\vx_s \mid \vx_t)
    = \underbrace{\log p_0(\vz)}_{\text{Gaussian prior}}
    + \underbrace{\log\big|\det J_{f_\mathcal{P}}\big|}_{\text{predictor}}
    + \underbrace{\sum_{\ell=1}^{L} \log\big|\det J_{f_\mathcal{T}^{(\ell)}}\big|}_{\text{transporter}}.
\end{equation}
Since the predictor is a diagonal affine coupling, $\log|\det J_{f_\mathcal{P}}| = -\sum_n \log \bm{\sigma}_\mathcal{P}^{(n)}$.
Each transporter block has a triangular Jacobian from the autoregressive structure.
Expanding the Gaussian prior term $\log p_0(\vz) = -\frac{1}{2}\|\vz\|^2 + \mathrm{const}$ recovers \cref{eq:ntm_nll} in the main text.

\paragraph{Relation to STARFlow.}
STARFlow~\citep{gu2025starflow} is a normalizing flow that models the \emph{marginal} image distribution $p(\vx)$ using the same deep-shallow architectural building blocks.
\methodname{} applies these blocks to a fundamentally different object: the \emph{conditional} distribution $p(\vx_s \mid \vx_t)$ at each denoising step.
This has two consequences.
First, the predictor in STARFlow operates spatially (causal attention over image patches within a single image), whereas in \methodname{} it operates over the trajectory dimension (non-causal attention across timestep levels), enabling cross-timestep reasoning.
Second, because the conditional $p(\vx_s \mid \vx_t)$ is simpler than the marginal $p(\vx)$---conditioning on $\vx_t$ already constrains the space of plausible images---\methodname{} requires fewer transporter blocks per step to achieve comparable expressiveness.

\subsection{Decomposition: Gaussian Denoising + Spatial Flow}
\label{sec:app_decomposition}

\methodname{} cleanly decomposes into two complementary components: a Gaussian denoising model in u-space (the predictor) and a non-Gaussian spatial transformation (the transporter).
We formalize this decomposition and show that without the transporter, \methodname{} reduces exactly to standard Gaussian diffusion.

\paragraph{Predictor alone $=$ Gaussian denoising.}
Suppose the transporter are absent, i.e., $f_\mathcal{T} = \mathrm{id}$ so that $\vu_s = \vx_s$.
The latent variable becomes $\vz = (\vx_s - \bm{\mu}_\mathcal{P}(\vx_t, t, s)) / \bm{\sigma}_\mathcal{P}(\vx_t, t, s)$, and the conditional distribution implied by the predictor is:
\begin{equation}
\label{eq:app_deep_only}
    p(\vx_s \mid \vx_t) = \mathcal{N}\!\big(\bm{\mu}_\mathcal{P}(\vx_t, t, s),\; \mathrm{diag}(\bm{\sigma}_\mathcal{P}^2)\big).
\end{equation}
This is a diagonal Gaussian---precisely the same family used by standard diffusion models and flow matching.
The \methodname{} loss in this case reduces to:
\begin{equation}
\label{eq:app_loss_no_shallow}
    \mathcal{L} = \sum_k \Big[\frac{1}{2}\Big\|\frac{\vx_{s_k} - \bm{\mu}_\mathcal{P}^{(k)}}{\bm{\sigma}_\mathcal{P}^{(k)}}\Big\|^2 + \sum_n \log \bm{\sigma}_\mathcal{P}^{(k,n)}\Big],
\end{equation}
which is the negative log-likelihood of a heteroscedastic Gaussian regression.
If $\bm{\sigma}_\mathcal{P}$ is further held fixed (not learned), minimizing over $\bm{\mu}_\mathcal{P}$ yields a weighted MSE loss, recovering the standard diffusion/flow-matching training objective up to a constant.

\paragraph{Transporter adds non-Gaussian expressiveness.}
The shallow autoregressive blocks introduce a \emph{nonlinear}, invertible change of coordinates $\vu_s = f_\mathcal{T}(\vx_s)$ before the Gaussian coupling is applied.
Even though the predictor still applies a Gaussian (affine) coupling in u-space, the \emph{implied distribution in x-space} is non-Gaussian because the inverse $\vx_s = f_\mathcal{T}^{-1}(\vu_s)$ is a nonlinear transformation of a Gaussian.
Formally:
\begin{equation}
\label{eq:app_nonGaussian}
    p(\vx_s \mid \vx_t) = \mathcal{N}\!\big(f_\mathcal{T}(\vx_s);\, \bm{\mu}_\mathcal{P},\, \mathrm{diag}(\bm{\sigma}_\mathcal{P}^2)\big) \cdot \big|\det J_{f_\mathcal{T}}(\vx_s)\big|.
\end{equation}
The Jacobian determinant $|\det J_{f_\mathcal{T}}|$ reweights the density to account for the nonlinear warping, allowing the model to represent multimodal, heavy-tailed, or skewed distributions in x-space even though u-space remains Gaussian.

\paragraph{Division of labor.}
In practice, the predictor captures the bulk of the denoising signal---predicting a good mean $\bm{\mu}_\mathcal{P}$ and an appropriate scale $\bm{\sigma}_\mathcal{P}$ in u-space---while the transporter learn the residual non-Gaussian structure via their spatial autoregressive coupling.
This division is efficient: the predictor uses a large Transformer backbone and concentrates most parameters on cross-timestep reasoning, while each transporter block is lightweight (2 layers) and handles local spatial dependencies.

\subsection{Effect of the FM Auxiliary Loss}
\label{sec:app_fm_aux}

When finetuning from a pretrained flow-matching model, \methodname{} adds a mean-alignment auxiliary loss (\cref{eq:fm_aux}):
\begin{equation}
    \mathcal{L}_\mathrm{aux} = \lambda \sum_{k} \big\| \bm{\mu}_\mathcal{P}^{(k)} - \bm{\mu}_\mathrm{FM}^{(k)} \big\|^2,
\end{equation}
where $\bm{\mu}_\mathrm{FM}$ is the denoising mean from a frozen copy of the pretrained backbone.
We analyze how this loss interacts with the \methodname{} likelihood objective.

\paragraph{Without $\mathcal{L}_\mathrm{aux}$ ($\lambda = 0$).}
The \methodname{} NLL objective jointly optimizes both the predictor ($\bm{\mu}_\mathcal{P}$, $\bm{\sigma}_\mathcal{P}$) and the transporter ($f_\mathcal{T}$).
In principle, the model is free to redistribute ``work'' between components: the predictor could learn a mean far from the pretrained solution if the transporter compensate.
In practice, this freedom can cause early-training instability, as the zero-initialized residual projection departs from the pretrained posterior before the transporter have learned a meaningful spatial mapping.

\paragraph{With $\mathcal{L}_\mathrm{aux}$ ($\lambda > 0$).}
The auxiliary loss anchors the predictor's mean prediction $\bm{\mu}_\mathcal{P}$ to the pretrained FM solution $\bm{\mu}_\mathrm{FM}$.
This has two consequences:
\begin{enumerate}
    \item \textbf{Stabilized u-space.} The u-space representation remains close to the pretrained model's latent space, providing a stable coordinate system in which the transporter can learn meaningful non-Gaussian corrections.
    \item \textbf{Non-Gaussian structure through $\bm{\sigma}_\mathcal{P}$ and $f_\mathcal{T}$.} Since $\bm{\mu}_\mathcal{P} \approx \bm{\mu}_\mathrm{FM}$, the non-Gaussian expressiveness must come from two sources: the learned scale $\bm{\sigma}_\mathcal{P}$ (which departs from the fixed Gaussian posterior variance $\sigma_\mathrm{post}$), and the transporter' spatial flow $f_\mathcal{T}$.
\end{enumerate}
In our experiments, we anneal $\lambda$ during training: starting at full strength to ensure stable initialization, then decaying so that the NLL objective can fine-tune the mean beyond the Gaussian approximation.

\paragraph{Limiting case: $\lambda \to \infty$.}
If $\lambda$ is very large, $\bm{\mu}_\mathcal{P}$ is forced to exactly match $\bm{\mu}_\mathrm{FM}$, and the predictor becomes equivalent to the pretrained Gaussian reverse step.
The \emph{only} source of non-Gaussian modeling is then the transporter.
The model reduces to: first, apply a spatial normalizing flow to transform $\vx_s$ into u-space; then, evaluate a fixed Gaussian posterior in u-space.
This is still more expressive than a standard diffusion model (which has no spatial flow), but less expressive than the full \methodname{} with learned $\bm{\mu}_\mathcal{P}$ and $\bm{\sigma}_\mathcal{P}$.

\subsection{Forward Transition Preserves Marginals}
\label{sec:app_marginal}

\begin{proposition}
\label{prop:marginal_preservation}
Let $q(\vx_t \mid \vx_0) = \mathcal{N}((1-t)\vx_0, t^2 \mI)$ and define the forward transition from time $s$ to $t$ ($s < t$) as:
\begin{equation}
    \vx_t = \alpha_{s,t}\,\vx_s + \sigma_{s,t}\,\bm{\epsilon}, \quad \bm{\epsilon} \sim \mathcal{N}(\mathbf{0}, \mI),
\end{equation}
with $\alpha_{s,t} = (1-t)/(1-s)$ and $\sigma_{s,t} = \sqrt{t^2 - \alpha_{s,t}^2 s^2}$.
If $\vx_s \sim q(\vx_s \mid \vx_0)$, then $\vx_t \sim q(\vx_t \mid \vx_0)$.
\end{proposition}

\begin{proof}
Since $\vx_s \sim \mathcal{N}((1-s)\vx_0,\, s^2 \mI)$, the transformed variable $\vx_t = \alpha_{s,t}\vx_s + \sigma_{s,t}\bm{\epsilon}$ is Gaussian with:
\begin{align}
    \mathbb{E}[\vx_t \mid \vx_0]
    &= \alpha_{s,t}\,(1-s)\vx_0
    = \frac{1-t}{1-s}(1-s)\vx_0
    = (1-t)\vx_0, \\
    \mathrm{Var}[\vx_t \mid \vx_0]
    &= \alpha_{s,t}^2 s^2 + \sigma_{s,t}^2
    = \alpha_{s,t}^2 s^2 + t^2 - \alpha_{s,t}^2 s^2
    = t^2.
\end{align}
Hence $\vx_t \sim \mathcal{N}((1-t)\vx_0, t^2 \mI) = q(\vx_t \mid \vx_0)$.
\end{proof}

\subsection{Reverse Posterior Coefficients}
\label{sec:app_posterior}

We derive the closed-form expressions for the coefficients $A(t,s)$, $B(t,s)$, and $C(t,s)$ in the reverse Gaussian posterior $q(\vx_s \mid \vx_t, \vx_0)$, used in the finetuning parameterization (\cref{eq:posterior}).

\begin{proposition}
Under the forward process $q(\vx_t \mid \vx_0) = \mathcal{N}((1-t)\vx_0, t^2\mI)$ with Markov transition (\cref{eq:markov_transition}), the reverse posterior is:
\begin{equation}
    q(\vx_s \mid \vx_t, \vx_0) = \mathcal{N}\!\big(A(t,s)\,\vx_t + B(t,s)\,\vx_0,\;\; C(t,s)^2\,\mI\big),
\end{equation}
where
\begin{align}
\label{eq:app_A}
    A(t,s) &= \frac{s^2(1-t)}{t^2(1-s)}, \\
\label{eq:app_B}
    B(t,s) &= \frac{(t-s)(t+s-2ts)}{t^2(1-s)}, \\
\label{eq:app_C}
    C(t,s)^2 &= \frac{s^2(t-s)(t+s-2ts)}{t^2(1-s)^2}.
\end{align}
\end{proposition}

\begin{proof}
By Bayes' rule, $q(\vx_s \mid \vx_t, \vx_0) \propto q(\vx_t \mid \vx_s)\, q(\vx_s \mid \vx_0)$, where:
\begin{align}
    q(\vx_s \mid \vx_0) &= \mathcal{N}\!\big((1{-}s)\vx_0,\; s^2 \mI\big), \\
    q(\vx_t \mid \vx_s) &= \mathcal{N}\!\big(\alpha_{s,t}\vx_s,\; \sigma_{s,t}^2 \mI\big).
\end{align}
The product of two Gaussians in $\vx_s$ is proportional to $\exp(-\frac{1}{2}\vx_s^\top \bm{\Lambda} \vx_s + \bm{\eta}^\top \vx_s)$ with precision and information vector:
\begin{align}
    \bm{\Lambda} &= \frac{1}{s^2}\mI + \frac{\alpha_{s,t}^2}{\sigma_{s,t}^2}\mI, \\
    \bm{\eta} &= \frac{(1-s)\vx_0}{s^2} + \frac{\alpha_{s,t}\vx_t}{\sigma_{s,t}^2}.
\end{align}
We first compute $\sigma_{s,t}^2 = t^2 - \alpha_{s,t}^2 s^2$.
Defining $D := t^2(1{-}s)^2 - s^2(1{-}t)^2$, we note that:
\begin{equation}
    D = \big[t(1{-}s) - s(1{-}t)\big]\big[t(1{-}s) + s(1{-}t)\big] = (t{-}s)(t{+}s{-}2ts),
\end{equation}
so $\sigma_{s,t}^2 = D / (1{-}s)^2$.

\textbf{Precision.}
\begin{equation}
    \Lambda = \frac{1}{s^2} + \frac{(1{-}t)^2/(1{-}s)^2}{D/(1{-}s)^2}
    = \frac{1}{s^2} + \frac{(1{-}t)^2}{D}
    = \frac{D + s^2(1{-}t)^2}{s^2 D}
    = \frac{t^2(1{-}s)^2}{s^2 D},
\end{equation}
where the last step uses $D + s^2(1{-}t)^2 = t^2(1{-}s)^2$.

\textbf{Posterior variance.}
\begin{equation}
    C^2 = \Lambda^{-1} = \frac{s^2 D}{t^2(1{-}s)^2} = \frac{s^2(t{-}s)(t{+}s{-}2ts)}{t^2(1{-}s)^2}.
\end{equation}

\textbf{Posterior mean.}
The information vector is $\bm{\eta} = (1{-}s)\vx_0 / s^2 + \alpha_{s,t}\vx_t / \sigma_{s,t}^2$.
Computing $\alpha_{s,t}/\sigma_{s,t}^2 = [(1{-}t)/(1{-}s)] \cdot (1{-}s)^2/D = (1{-}t)(1{-}s)/D$, the posterior mean is:
\begin{align}
    \bm{\mu}_\mathrm{post}
    &= C^2 \cdot \bm{\eta} \nonumber \\
    &= \frac{s^2 D}{t^2(1{-}s)^2} \left[\frac{(1{-}s)\vx_0}{s^2} + \frac{(1{-}t)(1{-}s)\vx_t}{D}\right] \nonumber \\
    &= \frac{D\,\vx_0}{t^2(1{-}s)} + \frac{s^2(1{-}t)\,\vx_t}{t^2(1{-}s)} \nonumber \\
    &= A(t,s)\,\vx_t + B(t,s)\,\vx_0.
\end{align}
In the finetuning recipe (\cref{sec:from_pretrained}), $\vx_0$ is replaced by the predicted clean sample $\hat{\vx}_0 = \vx_t - t \cdot \vv_\theta$, where $\vv_\theta$ is the pretrained velocity prediction.
\end{proof}

\paragraph{Numerical values for a 4-step schedule.}
\cref{tab:app_posterior_values} provides representative values of $A$, $B$, and $C$ for the default 4-step schedule used in our experiments.

\begin{table}[h]
\centering
\caption{Posterior coefficients for a representative 4-step schedule.}
\label{tab:app_posterior_values}
\small
\begin{tabular}{cccccc}
\toprule
Step & $t$ & $s$ & $A(t,s)$ & $B(t,s)$ & $C(t,s)$ \\
\midrule
1 & 1.000 & 0.754 & 0.140 & 0.614 & 0.371 \\
2 & 0.754 & 0.509 & 0.271 & 0.496 & 0.362 \\
3 & 0.509 & 0.263 & 0.470 & 0.362 & 0.297 \\
4 & 0.263 & 0.020 & 0.948 & 0.049 & 0.049 \\
\bottomrule
\end{tabular}
\end{table}

\subsection{Trajectory Covariance Matrix}
\label{sec:app_covariance}

The self-refinement step (\cref{eq:refinement}) uses a trajectory covariance matrix $\bm{S}$ whose $(i,j)$-th entry is the covariance between $\vx_{t_i}$ and $\vx_{t_j}$ conditioned on $\vx_0$.

\begin{proposition}
\label{prop:traj_cov}
Under the Markov forward process, for any two timesteps $t_i, t_j$ in the trajectory:
\begin{equation}
\label{eq:app_cov}
    \mathrm{Cov}(\vx_{t_i}, \vx_{t_j} \mid \vx_0) = \frac{\min(t_i, t_j)^2 \big(1 - \max(t_i, t_j)\big)}{1 - \min(t_i, t_j)} \cdot \mI.
\end{equation}
\end{proposition}

\begin{proof}
Without loss of generality, assume $t_i \leq t_j$ (so $\min = t_i$, $\max = t_j$).
Write $\vx_{t_i} = (1{-}t_i)\vx_0 + \bm{\xi}_i$ where $\bm{\xi}_i \sim \mathcal{N}(\mathbf{0}, t_i^2 \mI)$ is the noise component.
By the Markov property (\cref{eq:markov_transition}):
\begin{equation}
    \vx_{t_j} = \alpha_{t_i, t_j}\,\vx_{t_i} + \sigma_{t_i, t_j}\,\bm{\epsilon}, \quad \bm{\epsilon} \sim \mathcal{N}(\mathbf{0}, \mI) \text{ independent of } \vx_{t_i}.
\end{equation}
Therefore:
\begin{align}
    \mathrm{Cov}(\vx_{t_i}, \vx_{t_j} \mid \vx_0)
    &= \mathrm{Cov}(\bm{\xi}_i,\; \alpha_{t_i, t_j}\bm{\xi}_i + \sigma_{t_i, t_j}\bm{\epsilon}) \nonumber \\
    &= \alpha_{t_i, t_j}\,\mathrm{Var}(\bm{\xi}_i) \nonumber \\
    &= \frac{1 - t_j}{1 - t_i} \cdot t_i^2 \cdot \mI
    = \frac{t_i^2(1 - t_j)}{1 - t_i} \cdot \mI.
\end{align}
Since $t_i = \min(t_i, t_j)$ and $t_j = \max(t_i, t_j)$, this matches \cref{eq:app_cov}.
The diagonal case ($i = j$) gives $\mathrm{Var}(\vx_{t_i} \mid \vx_0) = t_i^2$, consistent with the formula via $\lim_{t_j \to t_i} t_i^2(1-t_j)/(1-t_i) = t_i^2$.
\end{proof}

\paragraph{Self-refinement update.}
Using this covariance, the self-refinement (\cref{eq:refinement}) applies a covariance-weighted gradient step:
\begin{equation}
    \hat{\vx} \leftarrow \frac{\hat{\vx} - \bm{S} \,\nabla_{\hat{\vx}} \mathcal{L}_{\text{\methodname{}}}}{1 - \vt},
\end{equation}
where the matrix-vector product $\bm{S}\,\nabla_{\hat{\vx}} \mathcal{L}$ couples the gradient across timesteps according to their noise correlation.
This is more effective than per-step independent correction (which would use only the diagonal $\mathrm{Var}(\vx_{t_i}) = t_i^2$), because correcting an error at one timestep propagates correlated corrections to all other timesteps.

\section{Algorithm Pseudocode}
\label{sec:app_algorithms}

\subsection{Training}
\label{sec:app_train}

\begin{algorithm}[h]
\caption{\methodname{} Training (single iteration)}
\label{alg:training}
\begin{algorithmic}[1]
\REQUIRE Clean data $\vx_0$, condition $\vy$ (text or class label), number of steps $T$, noise range $[t_\mathrm{min}^{\mathrm{lo}}, t_\mathrm{min}^{\mathrm{hi}}]$
\STATE Sample per-example minimum noise: $t_\mathrm{min} \sim \mathrm{Uniform}[t_\mathrm{min}^{\mathrm{lo}}, t_\mathrm{min}^{\mathrm{hi}}]$
\STATE Compute shifted timestep schedule: $(t_0, t_1, \ldots, t_T)$ with $t_0 = t_\mathrm{min}$
\STATE \textbf{Forward trajectory:} for $k = 0, \ldots, T{-}1$:
\STATE \quad $\bm{\epsilon}_k \sim \mathcal{N}(\mathbf{0}, \mI)$
\STATE \quad $\vx_{t_{k+1}} = \alpha_{t_k, t_{k+1}} \vx_{t_k} + \sigma_{t_k, t_{k+1}} \bm{\epsilon}_k$
\STATE \textbf{Transporter} (spatial AR flow): $\vu_{t_k} = f_\mathcal{T}(\vx_{t_k})$ for all $k$, accumulate $\log|\det J_{f_\mathcal{T}}|$
\STATE \textbf{Predictor} (trajectory coupling): for each consecutive pair $(t_{k+1}, t_k)$:
\STATE \quad $(\bm{\mu}_\mathcal{P}^{(k)}, \bm{\sigma}_\mathcal{P}^{(k)}) = \mathrm{DeepBlock}(\vu_{t_{k+1}}, t_{k+1}, t_k, \vy)$
\STATE \quad $\vz_k = (\vu_{t_k} - \bm{\mu}_\mathcal{P}^{(k)}) / \bm{\sigma}_\mathcal{P}^{(k)}$
\STATE \textbf{NTM loss}: $\mathcal{L}_{\text{\methodname{}}} = \sum_{k=1}^{T} \big[\frac{1}{2}\|\vz_k\|^2 + \sum_n \log \bm{\sigma}_\mathcal{P}^{(k,n)}\big] - \sum_\ell \log|\det J_{f_\mathcal{T}^{(\ell)}}|$
\STATE (\textit{Optional}) \textbf{FM auxiliary loss}: $\mathcal{L}_\mathrm{aux} = \lambda \sum_k \|\bm{\mu}_\mathcal{P}^{(k)} - \bm{\mu}_\mathrm{FM}^{(k)}\|^2$
\STATE Update $\theta$ via $\nabla_\theta (\mathcal{L}_{\text{\methodname{}}} + \mathcal{L}_\mathrm{aux})$
\end{algorithmic}
\end{algorithm}

\subsection{Sampling}
\label{sec:app_sample}

\begin{algorithm}[h]
\caption{\methodname{} Sampling}
\label{alg:sampling}
\begin{algorithmic}[1]
\REQUIRE Condition $\vy$, number of steps $T$, guidance scale $w$, schedule $(t_0, t_1, \ldots, t_T)$
\STATE Sample initial noise: $\hat{\vu}_{t_T} \sim \mathcal{N}(\mathbf{0}, \mI)$
\STATE \textbf{Predictor reverse} (parallel over spatial, sequential over $k$):
\FOR{$k = T, T{-}1, \ldots, 1$}
    \STATE $\vz_k \sim \mathcal{N}(\mathbf{0}, \mI)$
    \STATE $(\bm{\mu}_\mathcal{P}^{(k)}, \bm{\sigma}_\mathcal{P}^{(k)}) = \mathrm{DeepBlock}(\hat{\vu}_{t_k}, t_k, t_{k-1}, \vy)$
    \STATE (\textit{If CFG}) Apply guidance: $(\bm{\mu}_\mathcal{P}^{(k)}, \bm{\sigma}_\mathcal{P}^{(k)}) \leftarrow \mathrm{CFG}(\cdot, w)$
    \STATE $\hat{\vu}_{t_{k-1}} = \vz_k \cdot \bm{\sigma}_\mathcal{P}^{(k)} + \bm{\mu}_\mathcal{P}^{(k)}$
\ENDFOR
\STATE \textbf{Transporter inverse} (sequential AR decoding with KV-cache):
\STATE \quad $\hat{\vx}_{t_0} = f_\mathcal{T}^{-1}(\hat{\vu}_{t_0})$
\STATE (\textit{Optional}) \textbf{Self-refinement}: apply \cref{alg:refinement}
\STATE (\textit{Optional}) \textbf{Learned denoiser}: $\hat{\vx}_0 = D_\phi(\hat{\vu}, \vy, \vt)$
\STATE Decode: $\mathrm{image} = \mathrm{VAE}.\mathrm{decode}(\hat{\vx}_{t_0})$
\end{algorithmic}
\end{algorithm}

\subsection{Trajectory Self-Refinement}
\label{sec:app_refine}

\begin{algorithm}[h]
\caption{Trajectory Self-Refinement}
\label{alg:refinement}
\begin{algorithmic}[1]
\REQUIRE Generated trajectory $\hat{\vx} = (\hat{\vx}_{t_0}, \ldots, \hat{\vx}_{t_T})$, frozen \methodname{} model, schedule $(t_0, \ldots, t_T)$
\STATE Enable gradients w.r.t.\ $\hat{\vx}$
\STATE Forward pass through \methodname{}: compute $\mathcal{L}_{\text{\methodname{}}}(\hat{\vx})$
\STATE Compute gradient: $\vg = \nabla_{\hat{\vx}} \mathcal{L}_{\text{\methodname{}}}$
\STATE (\textit{Optional}) Percentile-based gradient clipping on $\vg$
\STATE Compute trajectory covariance: $[\bm{S}]_{ij} = \min(t_i, t_j)^2(1 - \max(t_i, t_j)) / (1 - \min(t_i, t_j))$
\STATE Covariance-weighted correction: $\hat{\vx} \leftarrow \hat{\vx} - \bm{S}\,\vg$ \quad \COMMENT{couples gradients across timesteps}
\STATE Normalize to clean domain: $\hat{\vx} \leftarrow \hat{\vx}\, /\, (1 - \vt)$
\RETURN $\hat{\vx}$
\end{algorithmic}
\end{algorithm}

\section{Implementation Details}
\label{sec:app_implementation}

\subsection{Model Architecture}
\label{sec:app_architecture}

\cref{tab:app_architecture} summarizes the architectural specifications of the \methodname{} models used in our experiments.

\begin{table}[h]
\centering
\small
\caption{Architectural specifications of \methodname{} models.}
\label{tab:app_architecture}
\setlength{\tabcolsep}{5pt}
\begin{tabular}{lcc}
\toprule
 & \textbf{From scratch} & \textbf{Finetuned} \\
\midrule
Hidden dimension & 3072 & 3072 \\
Number of blocks & 3 & 3 \\
Layers per block & [4, 4, 24] & [4, 4, 24] \\
\quad Transporter & Blocks 1--2 (4 layers each) & Blocks 1--2 (4 layers each) \\
\quad Predictor & Block 3 (24 layers) & FLUX.2-klein (4B) \\
Patch size & 1 & 2 \\
KV heads & 8 & 8 \\
Positional encoding & 2D RoPE & 2D RoPE \\
Transporter scan order & Alternating (L$\to$R, R$\to$L) & Alternating (L$\to$R, R$\to$L) \\
\midrule
Pretrained backbone & --- & FLUX.2-klein (4B) \\
Denoising mode & --- & \texttt{true\_reverse} \\
\texttt{no\_delta\_mean} & --- & 1 ($\bm{\mu}_\mathcal{P} = \bm{\mu}_\mathrm{post}$) \\
\bottomrule
\end{tabular}
\end{table}

\paragraph{Predictor.}
In the from-scratch setting, the predictor is a standard non-causal Transformer that processes all timestep levels in parallel.
It takes as input the u-space representations $(\vu_{t_0}, \ldots, \vu_{t_T})$, concatenated with text embeddings $\vy$, and predicts per-step coupling parameters $(\bm{\mu}_\mathcal{P}, \bm{\sigma}_\mathcal{P})$ via a linear projection layer.
Timestep conditioning is provided through additive sinusoidal embeddings.

In the finetuned setting, the predictor wraps a pretrained flow-matching backbone (FLUX.2).
The backbone's last hidden states are captured and fed to a zero-initialized projection layer $\mathrm{proj}_\mathrm{out}: \mathbb{R}^{d} \to \mathbb{R}^{2c}$ that outputs the residual corrections $(\delta_\mu, \delta_\sigma)$.
At initialization, $\mathrm{proj}_\mathrm{out} = \mathbf{0}$, so the predictor exactly reproduces the pretrained Gaussian posterior.

\paragraph{Transporter.}
Each transporter block is a TarFlow-style causal autoregressive flow with 2 Transformer layers.
Blocks alternate between identity and flip permutations (left-to-right and right-to-left scan directions) for better spatial mixing.
At the highest noise level ($t \approx 1$), the transporter are skipped (identity transform) since the input is nearly isotropic Gaussian and the spatial AR coupling would be uninformative.

\subsection{Training Hyperparameters}
\label{sec:app_hyperparams}
The hyperparameters are listed in \Cref{tab:app_hyperparams}.

\begin{table}[h]
\centering
\caption{Training hyperparameters.}
\label{tab:app_hyperparams}
\small
\setlength{\tabcolsep}{5pt}
\begin{tabular}{lcc}
\toprule
 & \textbf{From Scratch (ImageNet)} & \textbf{Finetuned} \\
\midrule
Optimizer & AdamW & AdamW \\
$(\beta_1, \beta_2)$ & $(0.9, 0.95)$ & $(0.9, 0.95)$ \\
Weight decay & $10^{-4}$ & $10^{-4}$ \\
Peak learning rate & $10^{-4}$ & $5 \times 10^{-5}$ \\
Minimum learning rate & $10^{-6}$ & $10^{-6}$ \\
LR schedule & Cosine with warmup & Cosine with warmup \\
Precision & bfloat16 & bfloat16 \\
Distributed strategy & FSDP2 & FSDP2 \\
\midrule
Denoising steps ($T$) & 4 & 4 \\
$t_\mathrm{min}$ range & $\mathrm{Uniform}[0.0, 0.05]$ & $\mathrm{Uniform}[0.0, 0.05]$ \\
CFG dropout & 10\% & 10\% \\
\midrule
FM aux loss weight ($\lambda$) & --- & 2.5 \\
FM aux loss type & --- & MSE \\
$\lambda$ annealing & --- & Cosine decay \\
\bottomrule
\end{tabular}

\end{table}

\subsection{Denoiser Architecture}
\label{sec:app_denoiser}

The learned denoiser $g_\phi$ (\cref{sec:denoiser}) is a lightweight Transformer that takes the predictor output $\vu_{t_0}$ at the cleanest level as input and produces a denoised image $\hat{\vx}_0^{\mathrm{den}}$ in a single forward pass.
Since the trajectory is Markov, $\vu_{t_0}$ contains all the information needed to deterministically predict the clean output.
\begin{itemize}[leftmargin=*]
    \item \textbf{Position encoding}: 2D rotary embeddings over spatial (row, column) dimensions.
    \item \textbf{Attention}: Full non-causal attention over all spatial positions.
    \item \textbf{Conditioning}: Text embeddings $\vy$ are concatenated to the input sequence.
    \item \textbf{Output}: A single predicted clean image in patch space.
    \item \textbf{Training}: After the main \methodname{} model converges, the frozen model generates targets via trajectory score denoising, and $g_\phi$ is trained with MSE loss (\cref{eq:denoiser_loss}).
\end{itemize}

\subsection{Timestep Schedule}
\label{sec:app_schedule}

We use a shifted timestep schedule~\citep{esser2024scaling} that adapts to the input sequence length.
Given $T$ denoising steps, the base schedule is:
\begin{equation}
    \tilde{\sigma}_k = \frac{k}{T}, \quad k = 1, \ldots, T,
\end{equation}
which is then shifted via a sequence-length-dependent parameter $\mu$:
\begin{equation}
\label{eq:app_shift}
    \sigma_k = \frac{e^{\mu}}{e^{\mu} + 1/\tilde{\sigma}_k - 1}, \quad \mu = 0.5 + 0.65 \cdot \frac{L_\mathrm{seq} - 256}{4096 - 256},
\end{equation}
where $L_\mathrm{seq}$ is the spatial sequence length (number of patches).
The final schedule is $(t_\mathrm{min}, \sigma_1, \ldots, \sigma_T)$ in ascending order, with $t_\mathrm{min}$ drawn per-sample from $\mathrm{Uniform}[t_\mathrm{min}^{\mathrm{lo}}, t_\mathrm{min}^{\mathrm{hi}}]$ for robustness across noise levels during training.

\subsection{Classifier-Free Guidance}
\label{sec:app_cfg}

At inference, we use a logits-guided formulation of classifier-free guidance~\citep{ho2022classifier} that operates on the coupling parameters rather than the predicted sample.
Given conditional predictions $(\bm{\mu}_c, \bm{\sigma}_c)$ and unconditional predictions $(\bm{\mu}_u, \bm{\sigma}_u)$ from the predictor, the guided parameters are:
\begin{align}
\label{eq:app_cfg}
    s &= \Big(\frac{\bm{\sigma}_c}{\bm{\sigma}_u}\Big)^2 \quad \text{(clipped to } [0, 1]\text{)}, \\
    \bm{\sigma}_\mathrm{eff} &= \frac{\bm{\sigma}_c}{\sqrt{1 + w - w \cdot s}}, \\
    \bm{\mu}_\mathrm{eff} &= \frac{(1+w)\,\bm{\mu}_c - w \cdot s \cdot \bm{\mu}_u}{1 + w - w \cdot s},
\end{align}
where $w$ is the guidance scale.
This formulation is inspired by the logit-space interpretation: it corresponds to $(1{+}w) \log p_c - w \log p_u$ applied to the Gaussian coupling in u-space, which naturally adjusts both the mean and scale (unlike standard linear guidance that only modifies the mean).

\section{Evaluation Benchmarks}
\label{sec:app_benchmarks}

\subsection{GenEval}
\label{sec:app_geneval}

GenEval~\citep{ghosh2023geneval} is a compositional text-to-image evaluation benchmark that tests fine-grained generation capabilities across six task categories:
\begin{itemize}
    \item \textbf{Single object}: generating a single named object correctly.
    \item \textbf{Two objects}: generating two distinct objects in the same scene.
    \item \textbf{Counting}: producing the exact number of objects specified.
    \item \textbf{Colors}: assigning the correct color to objects.
    \item \textbf{Position}: placing objects in the specified spatial relationship (e.g., ``left of'', ``above'').
    \item \textbf{Color attribution}: binding the correct color to the correct object when multiple colored objects are described.
\end{itemize}
Each task is scored by an object-detection model that verifies whether the specified objects, attributes, and relations are present.
The overall score is the average accuracy across all six tasks.
We generate 4 images per prompt and report the average detection rate.

\subsection{DPG-Bench}
\label{sec:app_dpg}

DPG-Bench~\citep{hu2024ella} (Dense Prompt Graph Benchmark) evaluates text-to-image alignment using long, detailed prompts that describe complex scenes with multiple entities, attributes, and relations.
Unlike GenEval which uses short compositional prompts, DPG-Bench tests whether models can faithfully follow dense, paragraph-length descriptions.
Evaluation is performed using a VQA model (BLIP-2) that answers questions about the generated image corresponding to each semantic element in the prompt.
The benchmark reports scores across five L1 categories:
\begin{itemize}
    \item \textbf{Attribute} (color, shape, size, texture, other)
    \item \textbf{Entity} (part, state, whole)
    \item \textbf{Global} (overall scene coherence)
    \item \textbf{Other} (counting, text rendering)
    \item \textbf{Relation} (spatial, non-spatial)
\end{itemize}
The overall DPG-Bench score is the average across all L1 categories, reported as a percentage.

\subsection{Class-Conditional ImageNet}
\label{sec:app_imagenet}

As a proof-of-concept for training \methodname{} from scratch, we evaluate on class-conditional ImageNet 256$\times$256 using the FAE latent space~\citep{gao2025one} ($16\times$ spatial compression, 32-dim latents).
\cref{tab:imagenet} reports FID-50K for \methodname{} at different step counts alongside representative baselines.

\begin{table}[h]
\centering
\small
\setlength{\tabcolsep}{4pt}
\caption{\textbf{Class-conditional ImageNet 256$\times$256} (FID-50K).
Steps: total sequential generation steps (denoising or autoregressive).
\methodname{} achieves competitive FID with significantly fewer steps than prior normalizing flows, using only the NLL training objective without distribution-level losses.}
\begin{tabular}{llccc}
\toprule
Method & Type & \#Params & Steps & FID$\downarrow$ \\
\midrule
DiT-XL/2~\citep{peebles2023dit} & DM & 675M & 250 & 2.27 \\
SiT-XL~\citep{ma2024sit} & DM & 675M & 250 & 2.06 \\
\midrule
LlamaGen~\citep{sun2024autoregressive} & AR & 3.1B & 256 & 2.18 \\
VAR~\citep{tian2024var} & AR & 2.0B & 10 & 1.73 \\
DART~\citep{gu2024dart} & AR & 820M & 16 & 3.82 \\
\midrule
TarFlow~\citep{zhainormalizing} & NF & 1.4B & 1024 & 5.56 \\
STARFlow (VAE)~\citep{gu2025starflow} & NF & 1.4B & 1024 & 2.40 \\
STARFlow (FAE)~\citep{gao2025one} & NF & 1.4B & 256 & 2.67 \\
\midrule
\methodname{} (Ours) & NF & 1.4B & 4 & 3.83 \\
\methodname{} (Ours) & NF & 1.4B & 8 & 3.24 \\
\methodname{} (Ours) & NF & 1.4B & 16 & 2.80 \\
\bottomrule
\end{tabular}
\label{tab:imagenet}
\end{table}

\methodname{} achieves 2.80 FID with 16 steps---comparable to STARFlow (FAE) at 2.67 which requires 256 autoregressive steps---demonstrating that the normalizing flow framework can produce competitive results with dramatically fewer sequential steps.
Notably, these results use \emph{only} the exact NLL training objective without any distribution-level losses (e.g., adversarial or perceptual).
Recent work~\citep{yin2024one,yin2024improved,yang2026representation} has shown that distribution-level finetuning can substantially boost few-step generators beyond their base performance;
\methodname{}'s stable exact-likelihood training makes it a natural candidate for such post-training enhancement, which we leave for future work.



\section{Additional Qualitative Results}
\label{sec:app_qualitative}
We show additional samples from our trained NTM models in \Cref{fig:placeholder}.
\begin{figure}[H]
    \centering
    \includegraphics[width=\linewidth]{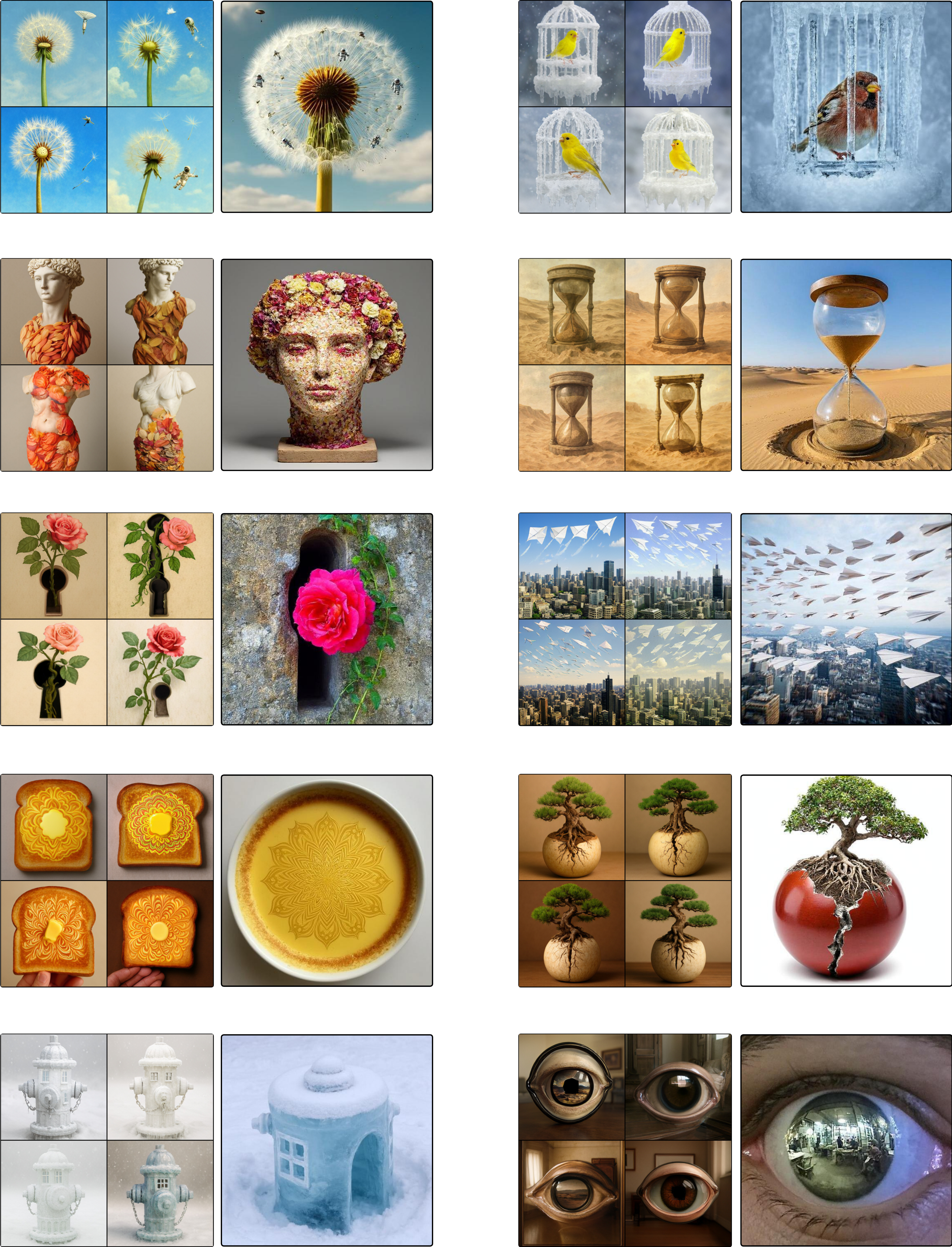}
    \caption{Additional examples from NTM trained from scratch (left) and fine-tuned from flow matching (right) under the same text prompts.}
    \label{fig:placeholder}
\end{figure}



\section{Broader Impact}
\label{sec:app_impact}

\methodname{} advances the state of the art in efficient image generation by enabling high-quality few-step sampling with exact likelihood.
While improved generative models have many beneficial applications---including creative tools, data augmentation, and scientific visualization---they also raise concerns around potential misuse for generating misleading or harmful content.
We believe that developing models with exact likelihood (as opposed to implicit or adversarial formulations) is a step toward more controllable and auditable generation, since the tractable density can support downstream applications such as anomaly detection and content verification.
We encourage the research community to develop complementary safeguards, including watermarking and provenance tracking, alongside advances in generative modeling.

\applefootnote{\textcolor{textgray}{\sffamily Apple and the Apple logo are trademarks of Apple Inc., registered in the U.S. and other countries and regions.}}

\end{document}